%% file: main.tex
\PassOptionsToPackage{table}{xcolor}  
\documentclass{article}

\usepackage{microtype}
\usepackage{graphicx}
\usepackage{subcaption}
\usepackage{booktabs} 

\usepackage{hyperref}

\usepackage{tikz}
\usepackage{makecell}

\usepackage[accepted]{icml2026}

\usepackage{amsmath}
\usepackage{amssymb}
\usepackage{mathtools}
\usepackage{amsthm}

\usepackage[capitalize,noabbrev]{cleveref} 

\theoremstyle{plain}
\newtheorem{theorem}{Theorem}[section]
\newtheorem{proposition}[theorem]{Proposition}
\newtheorem{lemma}[theorem]{Lemma}

\theoremstyle{definition}

\theoremstyle{remark}

\usepackage{multirow}
\usepackage[retain-zero-uncertainty=true]{siunitx}
\usepackage[version=4]{mhchem}
\usepackage{caption}
\usepackage{adjustbox}
\usepackage{comment}
\usepackage{nicefrac,xfrac}
\usepackage{placeins}
\usepackage{xcolor}

\newcommand{\KL}{D_{\mathrm{KL}}}

\DeclareMathOperator*{\argmin}{arg\,min}

\newcommand{\dd}{\mathrm{d}}

\icmltitlerunning{Efficient Training of Boltzmann Generators Using Off-Policy Log-Dispersion Regularization}

\begin{document}

\twocolumn[
\icmltitle{Efficient Training of Boltzmann Generators\\ Using Off-Policy Log-Dispersion Regularization}



\icmlsetsymbol{equal}{*}

\begin{icmlauthorlist}
\icmlauthor{Henrik Schopmans}{IAR,INT}
\icmlauthor{Christopher von Klitzing}{IAR}
\icmlauthor{Pascal Friederich}{IAR,INT}
\end{icmlauthorlist}

\icmlaffiliation{INT}{Institute of Nanotechnology, Karlsruhe Institute of Technology, Kaiserstr. 12, 76131 Karlsruhe, Germany}
\icmlaffiliation{IAR}{Institute for Anthropomatics and Robotics, Karlsruhe Institute of Technology, Kaiserstr. 12, 76131 Karlsruhe, Germany}

\icmlcorrespondingauthor{Pascal Friederich}{pascal.friederich@kit.edu}

\icmlkeywords{Machine Learning, ICML}

\vskip 0.3in
]



\printAffiliationsAndNotice{}  

\begin{abstract}
Sampling from unnormalized probability densities is a central challenge in
computational science. Boltzmann generators are generative models that enable
independent sampling from the Boltzmann distribution of physical systems at a
given temperature. However, their practical success depends on data-efficient
training, as both simulation data and target energy evaluations are costly. To
this end, we propose \textbf{off-policy log-dispersion regularization (LDR)}, a
novel regularization framework that builds on a generalization of the
log-variance objective. We apply LDR in the off-policy setting in combination
with standard data-based training objectives, without requiring additional
on-policy samples. LDR acts as a shape regularizer of the energy landscape by
leveraging additional information in the form of target energy labels. The
proposed regularization framework is broadly applicable, supporting unbiased or
biased simulation datasets as well as purely variational training without access
to target samples. Across all benchmarks, LDR improves both final performance
and data efficiency, with sample efficiency gains of up to one order of
magnitude.
\end{abstract}

\section{Introduction}
\label{sec:introduction}

Sampling from complex, high-dimensional probability distributions is a central
problem in computational physics, biology, and related areas. One such task is
sampling from the Boltzmann distribution of physical systems, where
$\tilde{p}_X(x) = \exp\!\left( \frac{-E(x)}{k_\text{B} T} \right)$ is an
unnormalized density, $x \in X$ is the configuration (e.g., particle positions
in $\mathbb{R}^{3N}$), $E$ an energy function, $T$ the temperature, and
$k_\text{B}$ the Boltzmann constant. The normalized distribution is $p_X(x) =
\nicefrac{\tilde{p}_X(x)}{\mathcal{Z}}$ with $\mathcal{Z} = \int_{X_0}
\tilde{p}_X(x) \, \dd x$, where the integral is taken over $X_0$ corresponding
to configurations with vanishing center of mass. The normalization constant
$\mathcal{Z}$ is usually unknown and intractable.

Sampling from the Boltzmann distribution enables the exploration of typical
configurations and estimating physically relevant expectations, e.g., in drug
discovery where protein-ligand binding free energies require extensive
exploration of thermodynamic ensembles
\cite{hollingsworthMolecularDynamicsSimulation2018}. Classical sampling
techniques such as molecular dynamics (MD)
\cite{alderStudiesMolecularDynamics1959} and Markov chain Monte Carlo (MCMC)
\cite{metropolisEquationStateCalculations1953} are asymptotically correct but
typically rely on long, correlated trajectories. As dimensionality and energetic
complexity increase, these methods can become prohibitively expensive,
particularly when energy evaluations are costly.

Generative modeling offers a complementary perspective by framing sampling as a
one-shot generation problem. Boltzmann generators
\cite{noeBoltzmannGeneratorsSampling2019a} enable independent sampling from
equilibrium distributions at a given temperature, making them an attractive
alternative or complement to trajectory-based methods. However, realizing a
practical advantage crucially depends on data-efficient training procedures.
Standard objectives for training Boltzmann generators typically rely on
equilibrium samples from the target distribution
\cite{tanAmortizedSamplingTransferable2025b,kleinTransferableBoltzmannGenerators2024a}
or require extensive energy evaluations in the variational energy-based training
setting
\cite{schopmansTemperatureAnnealedBoltzmannGenerators2025,klitzingLearningBoltzmannGenerators2025a},
both of which may be limiting in realistic settings.

To address these limitations, we introduce \textbf{off-policy log-dispersion
regularization (LDR)}, a general regularization framework for training Boltzmann
generators more efficiently. LDR builds on a generalization of the log-variance
objective, which has previously been used as a divergence for
on-policy\footnote{Using samples from the model distribution for training.}
variational training \cite{richterImprovedSamplingLearned2023}. In contrast, we
apply LDR off-policy on fixed datasets as a regularizer on top of standard
data-based objectives. LDR uses target energy labels to regularize the shape of
the learned energy landscape. Importantly, LDR can be evaluated over arbitrary
reference distributions and can therefore leverage biased simulation data or
auxiliary energy-labeled datasets that are otherwise difficult to incorporate.

The proposed framework applies to training on unbiased equilibrium data, biased
simulation data, and even purely energy-based variational training without
access to target samples. Across diverse benchmarks, we show that LDR
significantly improves final model performance and data efficiency.

We summarize our contribution as follows:

\begin{itemize}
    \item We adapt the log-variance objective from on-policy variational
    training to the off-policy setting, where it acts as a regularizer
    alongside standard data-based objectives for training on fixed, energy-labeled
    datasets.
    \item We generalize the log-variance objective to a family of
    log-dispersion objectives, including an L1 variant that is less sensitive
    to energy outliers and provides more stable optimization for Boltzmann
    generators.
    \item We show that log-dispersion regularization significantly increases
    final performance and data efficiency, both when training on unbiased and
    biased datasets.
    \item We show that log-dispersion regularization can also be adapted for
    energy-based variational training without target samples, improving the
    efficiency of the current state-of-the-art method, Constrained Mass
    Transport (CMT) \cite{klitzingLearningBoltzmannGenerators2025a}, by up to a
    factor of \num{10}.
\end{itemize}

\begin{figure*}[!htbp]
  \vskip 0.2in
  \begin{center}
  \centerline{\includegraphics{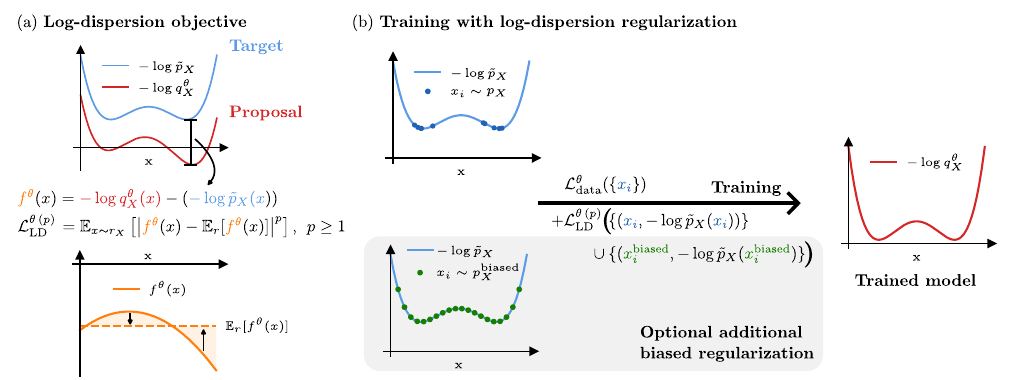}} \caption{(a) The
    log-dispersion objective minimizes the dispersion of $f^\theta(x)=-\log
    q_X^\theta(x) -(-\log \tilde{p}_X(x))$ around its mean, regularizing the
    shape of the proposal $q_X^\theta$ over the support of a reference
    distribution $r_X$. (b) When training off-policy, e.g., with a fixed
    dataset, log-dispersion alone is not a divergence because it does not
    constrain the proposal outside the support of $r_X$. We therefore use
    log-dispersion as a shape regularizer on top of data-based divergences to
    ensure global normalization. Since the reference distribution $r_X$ is
    flexible, LDR can be applied to target samples $x_i\sim p_X$ as well as
    biased samples $x_i^{\text{biased}}\sim p_X^{\text{biased}}$.}
  \label{fig:main_overview}
  \end{center}
\end{figure*} 

\section{Related Work}
 
Boltzmann generators were first introduced as an approach to obtain independent
samples from Boltzmann distributions by learning an invertible generative model
that supports exact likelihoods and importance reweighting
\cite{noeBoltzmannGeneratorsSampling2019a}. Closely related, Boltzmann emulators
focus on learning data-driven models of approximate equilibrium ensembles,
without exact reweighting guarantees
\cite{jingAlphaFoldMeetsFlow2024a,lewisScalableEmulationProtein2025,zhengPredictingEquilibriumDistributions2024}.

Since their introduction, many works have studied Boltzmann generators trained
from simulation data using standard data-based objectives,
including discrete normalizing flows\footnote{Here and throughout, ``discrete''
refers to normalizing flows parameterized as finite compositions of explicitly
invertible transformations, such as coupling-based or autoregressively
structured flows.}
\cite{midgleySE3EquivariantAugmented2023,tanScalableEquilibriumSampling2025a,rehmanEfficientRegressionBasedTraining2025},
flow matching approaches
\cite{kleinEquivariantFlowMatching2023b,rehmanFALCONFewstepAccurate2025b}, and
diffusion-based Boltzmann generators \cite{zhangEfficientUnbiasedSampling2025a}.
Several works investigate transferability, training on some molecular systems
and transferring to unseen ones to amortize simulation cost
\cite{jingTorsionalDiffusionMolecular2022,abdinDirectConformationalSampling2024,kleinTransferableBoltzmannGenerators2024a,tanAmortizedSamplingTransferable2025b}. 

Beyond data-based training, Boltzmann generators can also be learned
variationally using only target energy evaluations, without access to target
samples. Variational methods based on discrete normalizing flows have advanced
substantially in this regime
\cite{stimperResamplingBaseDistributions2022,midgleyFlowAnnealedImportance2022a,schopmansTemperatureAnnealedBoltzmannGenerators2025,klitzingLearningBoltzmannGenerators2025a},
with recent approaches currently achieving strong performance for variational
training in molecular systems \cite{klitzingLearningBoltzmannGenerators2025a}.
Diffusion-based variational counterparts have also emerged
\cite{liuAdjointSchrodingerBridge2025b,choiNonequilibriumAnnealedAdjoint2025,kimScalableEfficientTraining2025a},
but have so far been less competitive on molecular benchmarks; notably,
variational diffusion training on alanine dipeptide was only recently
demonstrated \cite{liuAdjointSchrodingerBridge2025b}, whereas flow-based
approaches reached this earlier and have since scaled to larger systems
\cite{midgleyFlowAnnealedImportance2022a,klitzingLearningBoltzmannGenerators2025a}.

Finally, several recent works explore training on data at elevated temperatures
and annealing toward lower temperatures, combining generative modeling with
tempering or progressive annealing schedules
\cite{dibakTemperatureSteerableFlows2022,wahlTRADETransferDistributions2025,schopmansTemperatureAnnealedBoltzmannGenerators2025,rissanenProgressiveTemperingSampler2025c,akhound-sadeghProgressiveInferenceTimeAnnealing2025a}.

A key limitation across related work is data efficiency: in data-based
approaches, the training signal typically comes from matching the sample
distribution and does not directly leverage target energy values typically
available through dataset generation. Likewise, in variational training, current
state-of-the-art methods optimize objectives that do not include a direct
training signal from target energy labels
\cite{schopmansTemperatureAnnealedBoltzmannGenerators2025,klitzingLearningBoltzmannGenerators2025a}.
Motivated by this gap, we introduce LDR, which improves data efficiency across a
broad range of training settings. Similar to our work,
\citet{vaitlFastUnifiedPath2023,vaitlPathGradientsFlow2025a} recently made an
important step by improving the forward KL with a path-gradient formulation,
yielding a lower-variance learning signal that incorporates gradients of the
target density. Compared to this approach, LDR (i) is a general off-policy
regularization framework that can be evaluated on arbitrary reference
distributions (including biased or auxiliary energy-labeled datasets), and (ii)
directly leverages target energies rather than target gradients. Finally, LDR is
complementary to path gradients and can be used on top of the path-gradient
forward KL as an additional energy-based regularizer.

\section{Method}
\label{sec:method}

To fit Boltzmann generators to a data distribution, they are typically trained
with data-based objectives that do not use target energy values, even though
these are usually available when constructing the dataset. Discrete normalizing
flows are commonly trained via the forward KL divergence, continuous normalizing
flows via flow matching, and diffusion models via score matching. While
effective, these objectives rely solely on samples and ignore structure provided
by the target energy landscape.

One way to incorporate target energy information is to directly compare the model
log-density $\log q_X^\theta(x)$ to the unnormalized target log density $\log
\tilde{p}_X(x)$. We define $f^\theta(x)$ as the difference of the negative log
densities (unnormalized log importance weights),

\begin{align}
f^\theta(x) &= -\log q_X^\theta(x) - (-\log \tilde{p}_X(x)) \notag \\
&= - \log q_X^\theta(x) -\frac{E(x)}{k_\text{B} T} \, .
\end{align}

At the optimum we have $q_X^\theta = p_X$, and thus $f^\theta(x) = \log
\mathcal{Z} = \text{const.}$ (see Figure~\ref{fig:main_overview}a for an
illustration). To obtain this optimum, the log-variance objective
\cite{richterVarGradLowVarianceGradient2020b} minimizes the variance of
$f^\theta(x)$ under a reference distribution $r_X(x)$:

\begin{align}
\mathcal{L}_{\mathrm{LV}}^{\theta} = \mathbb{E}_{r_X}\!\left[(f^\theta(x)-\mathbb{E}_{r_X}[f^\theta(x)])^2\right]
\end{align}

We generalize this log-variance objective to a family of log-dispersion
objectives that minimize the dispersion of $f^\theta(x)$ around its mean,

\begin{align}
\mathcal{L}_{\mathrm{LD}}^\theta
&= \mathbb{E}_{r_X}\!\left[\rho\!\left(f^\theta(x)-c^\theta\right)\right],
\qquad
c^\theta = \mathbb{E}_{r_X}\!\left[f^\theta(x)\right],
\label{eq:LDR}\\
\rho(u) &\ge 0,
\qquad
\rho(u)=0 \Longleftrightarrow u=0. \notag
\end{align}

In our main experiments, we focus on $\rho(u)=|u|^p$ ($p \ge 1)$, which yields
$\mathcal{L}_{\mathrm{LD}}^{\theta\,(p)} = \mathbb{E}_{r_X}\!\left[ \left|
f^\theta(x)-c^\theta \right|^p \right] $. We investigate LDR-L2 and LDR-L1,
obtained by choosing $p=2$ and $p=1$, respectively. LDR-L2 recovers the
log-variance objective and minimizes the second central moment (variance), while
LDR-L1 minimizes the first absolute central moment. We hypothesize that LDR-L1
is less sensitive to outliers and numerically more stable, an important property
for Boltzmann generators, where energy values can span multiple orders of
magnitude. We compare additional log-dispersion variants in
Appendix~\ref{sec:additional_LD}. Some related diffusion sampling works replace
the mean $c^\theta$ in Equation~\ref{eq:LDR} with a learnable scalar parameter
(trajectory balance objective) \cite{senderaImprovedOffpolicyTraining2024}.
Here, we simply use the batch-wise mean and backpropagate through both
$f^\theta(x)$ and $c^\theta$.

If $r_X(x)$ has full support, a member of the log-dispersion family defines a
divergence whose minimum is achieved if and only if $q_X^\theta = p_X$, assuming
that $q_X^\theta$ is properly normalized. However, in practice, the
configurations of a molecular system lie on a lower-dimensional manifold within
the full space. Thus, using a fixed dataset as $r_X(x)$ typically yields a
reference without full support, leaving the model unconstrained and arbitrarily
normalized outside the data manifold (see Appendix~\ref{appendix:no_forward_kl}
for an empirical demonstration and Appendix~\ref{appendix:theory} for a detailed
theoretical analysis). Instead of using a fixed reference, related work on
on-policy training with the log-variance objective chooses an adaptive on-policy
reference distribution, $r_X(x)=q_X^\theta(x)$, but this requires repeated
target energy evaluations and is prone to mode collapse in the finite-sample
regime. 

To avoid extra energy evaluations, we propose to deliberately demote the
log-dispersion objective to a regularizer, combining it with a standard
data-based objective:

\begin{align}
\mathcal{L}^\theta = \lambda_\text{data} \mathcal{L}_{\mathrm{data}}^\theta + \lambda_\text{LD} \mathcal{L}_{\mathrm{LD}}^\theta \, .
\end{align}

Figure~\ref{fig:main_overview}b illustrates the training procedure using this
log-dispersion regularization. The data-based term ensures correct normalization
and convergence to the correct target distribution over the full space, while
the log-dispersion term acts as a shape regularizer that aligns the proposal
with the target energy landscape by minimizing dispersion in the log importance
weights. This hybrid formulation provides an additional training signal through
the incorporation of target energy labels. Importantly, LDR can be evaluated
over arbitrary reference distributions and is not restricted to the target
distribution itself. Furthermore, in contrast to on-policy training with
log-dispersion objectives, our approach solely relies on a fixed dataset with
target energy labels, as it is usually already available after performing
unbiased or biased MD simulations. It does not have to use additional target energy
evaluations during training.

\section{Experiments}
\subsection{Experimental Setup}

\paragraph{Architecture}
One key property of Boltzmann generators is the ability to asymptotically unbias
the generated distribution using importance weights
$w(x)=\frac{\tilde{p}_X(x)}{q_X^\theta(x)}$. For some observable $h(x)$ it holds
\cite{martinoEffectiveSampleSize2017,noeBoltzmannGeneratorsSampling2019a}

\begin{align}
    \sum_{n=1}^N \frac{w(x_n)}{\sum_{i=1}^N w(x_i)} h\left(x_n\right) \xrightarrow[N \to \infty]{} \int h(x) p_X(x) \, \dd x \, . \label{eq:importance_sampling}
\end{align}

While generative approaches such as continuous normalizing flows and diffusion
models have shown strong performance, correcting their sampling bias via
importance sampling is typically computationally expensive
\cite{tanScalableEquilibriumSampling2025a,kleinTransferableBoltzmannGenerators2024a}.
Recent works focused on cheaper unbiasing of continuous normalizing flows or
diffusion models
\cite{zhangEfficientUnbiasedSampling2025a,pengFlowPerturbationAccelerate2025},
though no general solution is currently available. In contrast, discrete
normalizing flows provide exact and inexpensive likelihood evaluations by
construction. This property is especially critical in the context of Boltzmann
generators, where unbiased estimates of thermodynamic quantities are required.
For these reasons, we adopt discrete normalizing flows as the generative
backbone in this work.

For individual molecular systems, normalizing flows defined on internal
coordinate representations currently achieve state-of-the-art performance in
both data-based and variational training regimes
\cite{kimScalableNormalizingFlows2024a,klitzingLearningBoltzmannGenerators2025a}.
However, the choice of coordinate representation involves an inherent tradeoff:
internal coordinates are system-specific and non-unique, which limits their
transferability across different molecules. In contrast, Boltzmann generators
operating directly on Cartesian coordinates offer a more flexible and
transferable framework, but their performance still falls short of what can be
achieved using internal coordinate representations
\cite{kleinTransferableBoltzmannGenerators2024a,tanAmortizedSamplingTransferable2025b}.

We use a two-step experimental setup. First, we evaluate LDR using normalizing
flows in an internal coordinate representation
(Sections~\ref{sec:training_unbiased_datasets},
\ref{sec:training_biased_datasets}, and \ref{sec:variational_training}), which
are substantially cheaper to train and evaluate and yield stronger final
performance. We focus on controlled single-system benchmarks rather than
transferability, where the non-uniqueness of internal coordinates would be a
limiting factor. Second, to show that LDR is not tied to a particular coordinate
choice, we report experiments in Cartesian coordinates in
Section~\ref{sec:cart_coords}. For architecture details, see
Appendix~\ref{appendix:architecture}.

\paragraph{Target densities}
To benchmark LDR, we first use alanine dipeptide ($d=60$), which is a
well-studied benchmark system in prior work
\cite{dibakTemperatureSteerableFlows2022,stimperResamplingBaseDistributions2022,midgleyFlowAnnealedImportance2022a,tanAmortizedSamplingTransferable2025b}.
We further use the larger alanine
hexapeptide ($d=180$)
\cite{schopmansTemperatureAnnealedBoltzmannGenerators2025,klitzingLearningBoltzmannGenerators2025a}
to test the scalability of our method. Both systems have complex metastable
regions that occupy only a small part of the overall state space, making them
well suited as challenging benchmarks.

\input{tables/results_unbiased.tex}

\paragraph{Metrics} 
As metrics, we first use the negative log-likelihood defined as
$\text{NLL}=-\mathbb{E}_{x \sim p_X(x)}\left[\log q_X^\theta(x)\right]$. The NLL
is equivalent to the forward KL divergence up to an additive constant, and thus
provides a reliable signal for mode collapse
\cite{blessingELBOsLargeScaleEvaluation2024} and model performance.

We additionally report the effective sample size (ESS), which quantifies how
many independent draws from the target distribution $p_X$ would be required to
attain the same estimator variance as that obtained using samples from the flow
distribution $q_X^\theta$ \cite{martinoEffectiveSampleSize2017}. In practice, we
compute the reverse ESS from flow samples \cite{martinoEffectiveSampleSize2017}.
Since the reverse ESS is restricted to the support of $q_X^\theta$, a large ESS
can occur under mode collapse; we therefore always assess it jointly with NLL.

We introduce and report additional metrics in
Appendix~\ref{appendix:section:metrics}. In the main manuscript, consistent with
recent work
\cite{kleinTransferableBoltzmannGenerators2024a,vaitlPathGradientsFlow2025a}, we
focus on NLL and ESS, which we found most reliable for assessing relative
performance. However, NLL is not normalized, and its absolute scale is not
directly interpretable, so even small differences can be significant. We
therefore use NLL primarily to rank methods and refer to the metrics in the
appendix when additional nuance is required.

We further refer to Figure~\ref{fig:all_ramachandran_tica_overview} in
Appendix~\ref{sec:appendix:visualizations} for a visualization of the 2D
marginal densities of the main degrees of freedom of each system (Ramachandran
and TICA plots).

\subsection{Training on Unbiased Datasets} \label{sec:training_unbiased_datasets}
We start with the most straightforward setup, training on an unbiased dataset
that follows the target distribution and includes target energy labels,
$\mathcal D_n = \{(x_1,E(x_1)),\dots,(x_n,E(x_n))\}, \, x_i
\sim p_X$. Datasets used to train Boltzmann generators are
typically obtained from MD simulations, so target energy values already have to
be computed during dataset creation. We can therefore apply energy
regularization without extra cost, using $\mathcal D_n$ as the reference
distribution of the LD objective.

To test the effect of LDR in different data regimes, we train on both alanine
dipeptide and alanine hexapeptide using datasets of increasing size, i.e.
\num{1e6}, \num{2e6} and \num{5e6} samples. As a baseline, we use forward KL
training without regularization, and compare it against models trained with
LDR-L1 and LDR-L2 regularization.

In addition to the molecular benchmarks, we include a simple 2D Gaussian mixture
model. Here, we provide the path-gradient forward KL
\cite{vaitlFastUnifiedPath2023} as an additional baseline, which uses gradients
of the target density. For the molecular systems in internal coordinates, we do
not include the path-gradient forward KL baseline, since we did not observe
stable training in this setting; we include a detailed discussion and analysis
in Appendix~\ref{appendix:path_gradients}.

\input{tables/results_biased.tex}

\paragraph{Results} 
Our results for training on unbiased datasets are summarized in
Table~\ref{tab:results_unbiased} (full results can be found in
Appendix~\ref{appendix:extended_results}). Across all three tasks, we observe
significant improvements in both NLL and ESS when using LDR compared to the
baselines. For the Gaussian mixture task, LDR regularization allows stable
training on only \num{500} data points, while the performance without LDR
significantly degrades for this dataset size. For alanine dipeptide, LDR using
\num{1e6} samples outperforms training without LDR with \num{5e6} samples,
indicating more than \num{5} times higher data efficiency. For alanine
hexapeptide, we observe that LDR yields the largest improvements for \num{2e6}
and \num{5e6} training samples (ESS plot on the right of
Table~\ref{tab:results_unbiased}). For \num{1e6} samples, the model likely
remains too far from the target, leading to high-variance importance weights and
a noisy LD objective, making it least effective in this regime. Moreover, LDR-L1
performs better than LDR-L2 for hexapeptide, consistent with the L1 objective
being less sensitive to outliers. Appendix~\ref{sec:additional_LD} provides
additional support for the improved stability of LDR-L1 through an analysis of
gradient-norm variability during training.

\subsection{Training on Biased Datasets} \label{sec:training_biased_datasets}
Unbiased molecular dynamics simulations become prohibitively expensive when the
relevant configuration space is separated by large free-energy barriers or
characterized by slow collective modes. In such regimes, trajectories can remain
trapped in metastable regions for long times, making it infeasible to obtain
sufficient equilibrium samples from $p_X$ within reasonable computational
budgets.

By relaxing the requirement for unbiased samples, biased simulation techniques
accelerate exploration by modifying the sampling dynamics, trading exact
sampling from $p_X$ for improved state-space coverage
\cite{abramsEnhancedSamplingMolecular2013}. The resulting datasets are drawn
from a biased distribution $p_X^{\text{biased}}$ but are often orders of
magnitude cheaper to obtain than unbiased equilibrium data.

We train Boltzmann generators on biased simulation data using a two-stage
procedure. First, we perform standard data-based training, using the forward KL
on a biased dataset $\mathcal D_n^{\text{biased}} = \{x_i\}_{i=1}^n,\, x_i \sim
p_X^{\text{biased}}$. This yields an initial proposal distribution
$q_X^{\theta_1}$ that approximates the biased distribution. We use the same
pre-trained model $q_X^{\theta_1}$ for all methods.

Second, we generate samples $x_j \sim q_X^{\theta_1}$ and compute importance weights
$w(x_j)=\tilde p_X(x_j)/q_X^{\theta_1}(x_j)$. From this, we construct an approximate
equilibrium dataset $\mathcal D_m^{\text{IS}}$ via categorical resampling, i.e.,
by drawing samples with probabilities proportional to $w(x_j)$.

We then refine our initial biased proposal $q_X^{\theta_1}$ by training on
$\mathcal D_m^{\text{IS}}$. Here, we compare standard forward KL training on
$\mathcal D_m^{\text{IS}}$ with its LDR-L1-regularized and LDR-L2-regularized
variants. This form of self-refinement has been previously used in other
contexts
\cite{schopmansTemperatureAnnealedBoltzmannGenerators2025,tanAmortizedSamplingTransferable2025b}.
Since LDR can be evaluated over arbitrary reference distributions, we can use
not only $\mathcal D_m^{\text{IS}}$ for LDR, but also a mixture of $\mathcal
D_m^{\text{IS}}$ and the biased dataset $\mathcal D_n^{\text{biased}}$ (we use a
$50:50$ mixture). This leverages additional training signal from the biased
distribution, which is typically difficult to incorporate into Boltzmann
generator training.

We note that the importance-sampling step requires sufficient overlap between
$q_X^{\theta_1}$ and $p_X$; modes absent from the proposal cannot be recovered
by resampling. If a (small) unbiased dataset is available, one can instead skip
the importance sampling step and add arbitrary biased data to the reference
distribution of the LDR objective, while using the unbiased data for the
data-based objective. This allows incorporating biased distributions that do not
cover all modes. However, we do not explore this option here, and focus on the
importance sampling approach.

\paragraph{Dataset} To test our methodology for training on biased datasets, we
use short trajectories starting in the main modes of the energy landscape of
alanine dipeptide. The short length of each trajectory leads to high
correlations across trajectories, ultimately resulting in a heavily biased
dataset of \num{1e6} samples. Details can be found in
Appendix~\ref{appendix:datasets}.

In total, generating this biased dataset required only \num{1.1e6}
target evaluations and can be performed in less than \SI{200}{\second} on a regular
laptop. In contrast, typical MD simulations require
$>\num{1e8}$ steps for high-quality unbiased datasets for alanine dipeptide (our
ground truth MD simulation took $\sim 2$ days to simulate).

\input{tables/results_cmt.tex}

\paragraph{Results} 
Results for training on the biased dataset of alanine dipeptide are summarized
in Table~\ref{tab:results_biased} (full results can be found in
Appendix~\ref{appendix:extended_results}). Using only \num{1e5} importance
samples, LDR applied on $\mathcal D_m^{\text{IS}}$ substantially improves
performance compared to the non-regularized baseline. Applying regularization on
both $\mathcal D_m^{\text{IS}}$ and the original biased dataset $\mathcal
D_n^{\text{biased}}$ further improves performance, leading to a substantial gap.
LDR effectively leverages information from the biased data distribution, leading
to substantially improved refinement compared to standard
importance-sampling-based self-refinement. 

In total, LDR allows training of a well-performing Boltzmann generator using
only \num{1e5} + \num{1.1e6} (to construct the biased dataset) target
evaluations. To achieve similar performance, non-regularized training requires
more than \num{10} times more IS samples.

\subsection{Variational Training without Target Data} \label{sec:variational_training}
Next, we show that regularizing data-based training also improves efficiency in
the purely variational setting, where no target data (unbiased or biased) are
available and training uses only target energy evaluations. Several recent
methods train Boltzmann generators by annealing the model distribution toward
the target through a sequence of intermediate distributions $q_i$
\cite{schopmansTemperatureAnnealedBoltzmannGenerators2025,klitzingLearningBoltzmannGenerators2025a}.
To learn each intermediate distribution $q_i$, a buffer is built by reweighting
model samples to the current intermediate distribution. Standard data-based
training is then performed on each buffer until the next intermediate
distribution is constructed, so we can readily apply LDR, using the buffer as
the reference distribution.

Temperature-Annealed Boltzmann Generator (TA-BG)
\cite{schopmansTemperatureAnnealedBoltzmannGenerators2025} pre-trains with the
reverse Kullback-Leibler (KL) divergence at an elevated temperature and
subsequently anneals the model distribution using a number of intermediate
temperatures $q_{i} \propto \tilde p_X^{\alpha_i} $. Constrained Mass Transport
(CMT) \cite{klitzingLearningBoltzmannGenerators2025a} extends this idea and uses
a generalized annealing path of the form $q_{i} \propto q_0^{1-\lambda_i} \left(
\tilde p_X^{\alpha_i}\right)^{\lambda_i}$. The parameter schedule
$(\lambda_i)_i$ is chosen adaptively based on a trust-region constraint that
limits the KL divergence between subsequent intermediate distributions. This
guarantees sufficient overlap between successive intermediates and improves
performance compared to TA-BG. We provide a more detailed introduction to the
annealing-based variational framework used in TA-BG and CMT in
Appendix~\ref{appendix:variational_training_annealing}. As an additional
baseline, we include Flow Annealed Importance Sampling Bootstrap (FAB)
\cite{midgleyFlowAnnealedImportance2022a}, which was the first to successfully
learn alanine dipeptide variationally without mode collapse.

LDR can be applied to FAB, TA-BG, and CMT, as all use buffered off-policy
training. Since CMT represents the current state of the art, we apply LDR only
to CMT. For FAB and TA-BG, we run baseline experiments without LDR.

CMT provides a particularly suitable environment for applying LDR: by limiting
the KL-divergence between intermediate distributions, it enforces sufficient
overlap, which stabilizes the resulting log importance weights. This yields a
lower-variance learning signal for log-dispersion than, e.g., in our experiments
directly on unbiased datasets. We therefore view the combination of CMT and LDR as particularly promising for
scaling equilibrium sampling to larger systems, although this remains an active
area of research
\cite{tanScalableEquilibriumSampling2025a,tanAmortizedSamplingTransferable2025b}
and is beyond the scope of this work.

\paragraph{Results}
Table~\ref{tab:results_cmt} summarizes our results of applying LDR to
variational training with CMT (full results can be found in
Appendix~\ref{appendix:extended_results}). LDR substantially improves the
performance of CMT for all tested settings of total target evaluations. The gap
becomes especially pronounced when reducing the number of target evaluations. On
alanine dipeptide, CMT + LDR requires only \num{1e7} target evaluations to
achieve comparable performance to non-regularized CMT with \num{1e8} target
evaluations (see Table~\ref{tab:results_CMT_all}). This indicates an
approximately \num{10} times higher efficiency in terms of target evaluations.

CMT on alanine hexapeptide was originally reported with \num{4e8} target
evaluations. In this setting, LDR substantially improves final performance,
increasing ESS by more than \num{20} percentage points. When lowering the number
of target evaluations to \num{1e8}, vanilla CMT shows mode collapse (see
Figure~\ref{fig:all_ramachandran_tica_overview} in the appendix) and reaches an
ESS of only \SI{5.92}{\percent}. In contrast, LDR makes it possible to still
achieve stable training without mode collapse in this setting, achieving an ESS
of over \SI{30}{\percent}. We further note that FAB and TA-BG perform
considerably worse, even though they use $\geq 4$ times the number of target
evaluations.

Overall, LDR allowed us to significantly reduce the number of target evaluations
used in CMT while largely preserving final performance. We refer to
Appendix~\ref{sec:appendix:pushing_the_limits} for an ablation where we reduce
the number of target evaluations even further than what we report in our main
experiments.

\subsection{Training Using Cartesian Coordinates}
\label{sec:cart_coords}

To show that the gains from LDR are not tied to internal coordinate
representations, we additionally train Boltzmann generators directly on
Cartesian coordinates. To match the data regime of related works, we use
\num{1e5} training samples \cite{tanScalableEquilibriumSampling2025a}. We use an
autoregressive normalizing flow based on a transformer with causal masking
\cite{zhaiNormalizingFlowsAre2025a}, adapted to Boltzmann generators by
\citeauthor{tanScalableEquilibriumSampling2025a}
\yrcite{tanScalableEquilibriumSampling2025a}. Following
\citet{tanScalableEquilibriumSampling2025a}, we incorporate the symmetries via
data augmentation (random rotations and center-of-mass noise after centering),
and thus train the flow in an augmented space. When estimating unbiased
quantities via importance sampling, this requires a correction to map back to
the target space; we propose an improved version of this correction compared to
the one introduced by
\citeauthor{tanScalableEquilibriumSampling2025a}
\yrcite{tanScalableEquilibriumSampling2025a} (see
Appendix~\ref{appendix:theory:com_augm} for details).

We note that this discrete normalizing flow operating on Cartesian coordinates
is substantially faster than continuous normalizing flow or diffusion-based
counterparts used in related works. In Appendix~\ref{appendix:architecture}, we
provide a detailed inference-speed comparison demonstrating that TarFlow
achieves nearly a $4000\times$ higher sampling throughput than a representative
continuous normalizing flow baseline.

\input{tables/results_tarflow.tex}

\paragraph{Results} Table~\ref{tab:results_tarflow} summarizes our results on
Cartesian coordinates. First, we observe that our improved augmentation
correction improves results from the original formulation
\cite{tanScalableEquilibriumSampling2025a}. Furthermore, LDR outperforms
non-regularized forward KL by a substantial margin. Both LDR-L2 and LDR-L1
perform similarly well. We provide a detailed discussion of these results in
Appendix~\ref{appendix:extended_discussion_cart}. There, we also provide a
comparison with the path-gradient version of the forward KL as an additional
baseline \cite{vaitlFastUnifiedPath2023}. Our results show that LDR provides a
strong additional training signal from energy labels also in the
Cartesian-coordinate setting. While we focused on a controlled single-system
benchmark, future work can explore training with LDR in the transferable setting
\cite{tanAmortizedSamplingTransferable2025b}.

\section{Discussion}
\label{sec:discussion}
While we focus on normalizing flows due to their accuracy and efficient
importance sampling, LDR applies to any likelihood-based model class. It can
also be applied to diffusion-based Boltzmann generators by reformulating
Equation~\ref{eq:LDR} over joint diffusion paths rather than terminal marginals,
which upper-bounds the terminal objective
\cite{sanokowskiRethinkingTrainingDiffusion2025}. While variational on-policy
log-variance training for diffusion models is well-studied
\cite{richterImprovedSamplingLearned2023,senderaImprovedOffpolicyTraining2024},
using it as off-policy regularization for training diffusion models has not been
explored. However, we consider this outside the scope of the current manuscript.

Finally, we emphasize the strong dependence of final performance on dataset size
observed across all experiments. The largest unbiased datasets used in this work
(\num{5e6} samples) are considerably larger than those employed in related
studies. We consistently observe substantial performance improvements when
increasing dataset size beyond the regimes explored in prior work, both with and
without LDR. This suggests that dataset size is a critical factor for Boltzmann
generator training. Importantly, while LDR does not remove this dependence, it
mitigates it by increasing data efficiency: models trained with LDR experience
significantly improved metrics at smaller dataset sizes.

In the main experiments, we evaluated two variants of LDR, LDR-L1 and LDR-L2.
Across all benchmarks, both variants yield comparable and substantial
performance gains, indicating that the benefits of LDR are robust to the
specific choice of dispersion measure. We observe a slight advantage of LDR-L1
in the unbiased training of alanine hexapeptide, which we attribute to the
heavy-tailed distribution of importance weights in this high-dimensional energy
landscape. By penalizing absolute rather than squared deviations, LDR-L1 is less
sensitive to outliers and therefore provides more stable optimization in this
regime. Appendix~\ref{sec:additional_LD} contains additional experiments on
other members of the log-dispersion family; we generally find that dispersion
penalties with weaker tail sensitivity perform favorably. Overall, however, the
performance difference between LDR-L1 and LDR-L2 remains small, suggesting that
either variant is a reasonable default choice in practice.

\raggedbottom
\paragraph{Limitations} 
A limitation of log-dispersion regularization is the introduction of an
additional hyperparameter in the form of the relative loss weighting between LDR and
the data-based objective. However, we observe only moderate sensitivity to this
choice: as shown in Appendix~\ref{appendix:loss_weight_sensitivity},
performance remains stable over a range of loss weights.

A second limitation is that applying LDR with reference distributions far from
the target may cause unstable training. Although this did not occur in our
experiments, understanding LDR's behavior under increasingly biased references
is an important direction for future work.

\section{Conclusion}
\label{sec:conclusion}

We introduced \textbf{off-policy log-dispersion regularization (LDR)}, a simple
yet powerful framework for incorporating target energy information into the
training of Boltzmann generators. By generalizing the log-variance objective and
using it as an off-policy shape regularizer, LDR provides an additional training
signal without requiring extra on-policy samples or additional target energy evaluations.

Across a wide range of settings, including training on unbiased equilibrium
data, biased simulation datasets, and purely variational training without access
to target samples, LDR consistently improves both final performance and data
efficiency. In several benchmarks, these gains translate into order-of-magnitude
reductions in the number of samples or target energy evaluations required to
reach a given performance level. Importantly, the framework is broadly
applicable and easy to integrate into existing pipelines.

\section*{Reproducibility Statement}
The source code to reproduce our experiments can be found on GitHub:
\url{https://github.com/aimat-lab/Log-Dispersion-Regularization}. This
repository also contains information on how to obtain the ground truth datasets
necessary for training and evaluation.

\section*{Acknowledgments}
H.S. acknowledges financial support by the German Research Foundation (DFG)
through the Research Training Group 2450 “Tailored Scale-Bridging Approaches to
Computational Nanoscience”. P.F. acknowledges funding from the Klaus Tschira
Stiftung gGmbH (SIMPLAIX project) and the pilot program Core-Informatics of the
Helmholtz Association (KiKIT project). Parts of this work were performed on the
HoreKa supercomputer funded by the Ministry of Science, Research and the Arts
Baden-Württemberg and by the Federal Ministry of Education and Research. 
The authors gratefully acknowledge the Gauss Centre
for Supercomputing e.V. (www.gauss-centre.eu) for funding this project by
providing computing time through the John von Neumann Institute for Computing
(NIC) on the GCS Supercomputer JUWELS at Jülich Supercomputing Centre (JSC).

\section*{Impact Statement}
This paper presents work whose goal is to advance the field of machine learning.
There are many potential societal consequences of our work, none of which we
feel must be specifically highlighted here.

\clearpage
\flushbottom

\bibliography{main}
\bibliographystyle{icml2026}

\newpage
\appendix
\onecolumn

\clearpage 

\section{Theory}
\label{appendix:theory}

We consider the log-dispersion objective
\[
\mathcal L_{\mathrm{LD}}^\theta(r_X;\rho)
=
\mathbb E_{r_X}\!\left[\rho\!\left(f^\theta(x)-c^\theta\right)\right],
\qquad
c^\theta=\mathbb E_{r_X}\!\left[f^\theta(x)\right],
\]
where the corresponding expectations are assumed to be well defined and
$\rho$ satisfies
\[
\rho(u) \ge 0,
\qquad
\rho(u)=0 \Longleftrightarrow u=0.
\]

\subsection{Optimum of Training with Log-Dispersion Objectives}
\begin{proposition}[Log-dispersion regularization as a divergence]
Let $r_X$ be a fixed reference distribution with full support on $X$, i.e.,
$r_X(x)>0$ for all $x\in X$. Let $\tilde p_X(x)>0$ be an unnormalized target
density and let $q_X^\theta$ be a normalized model density with
$q_X^\theta(x)>0$ for all $x$. Then $\mathcal
L_{\mathrm{LD}}^\theta(r_X;\rho)\ge 0$, and
\[
\mathcal L_{\mathrm{LD}}^\theta(r_X;\rho)=0
\quad\Longleftrightarrow\quad
q_X^\theta(x)=\frac{\tilde p_X(x)}{\mathcal Z}=p_X(x).
\]
In particular, $\mathcal L_{\mathrm{LD}}^\theta(r_X;\rho)$ is a divergence whose
unique minimum is attained at the target distribution.
\end{proposition}

\begin{proof}
Nonnegativity is immediate since $\rho(u)\ge 0$.

If $\mathcal L_{\mathrm{LD}}^\theta(r_X;\rho)=0$, then
$\rho(f^\theta(x)-c^\theta)=0$, which implies
$f^\theta(x)=c^\theta$. Because $r_X$ has full support, this equality
holds everywhere on $X$. Thus,
\[
\log\frac{\tilde p_X(x)}{q_X^\theta(x)}=c^\theta
\quad\Longleftrightarrow\quad
q_X^\theta(x)=e^{-c^\theta}\tilde p_X(x).
\]
Normalization of $q_X^\theta$ yields
\[
1=\int q_X^\theta(x)\,dx
=e^{-c^\theta}\int \tilde p_X(x)\,dx
=e^{-c^\theta}\mathcal Z,
\]
and therefore $q_X^\theta(x)=\tilde p_X(x)/\mathcal Z=p_X(x)$.

Conversely, if $q_X^\theta=p_X$, then $f^\theta(x)=\log\mathcal Z$ is constant, and
hence $\mathcal L_{\mathrm{LD}}^\theta(r_X;\rho)=0$.
\end{proof}

\subsection{Log-Dispersion Using Reference Distributions without Full Support}

The divergence property of the log-dispersion objective established above relies
critically on the assumption that the reference distribution $r_X$ has full
support on $X$. In this work, however, we want to use fixed datasets,
corresponding to a reference distribution whose support is restricted to a
subset of configuration space. In this setting, the divergence property no
longer holds.

\begin{proposition}[LD without full support is not a divergence]
\label{prop:ldr_nofullsupport}
Let $r_X$ be a reference distribution whose support $\operatorname{supp}(r_X)
\subsetneq X$ is a strict subset of the configuration space. Then
$\mathcal L_{\mathrm{LD}}^\theta(r_X;\rho)$ is not a divergence on the space of
normalized densities on $X$. In particular, there exist normalized
densities $q_X^\theta \neq p_X$ such that
\[
\mathcal L_{\mathrm{LD}}^\theta(r_X;\rho)=0.
\]
\end{proposition}

\begin{proof}
If $\operatorname{supp}(r_X) \subsetneq X$, then the condition
\[
\mathcal L_{\mathrm{LD}}^\theta(r_X;\rho)=0
\]
implies only that $f^\theta(x)=c^\theta$ holds $r_X$-almost surely, i.e., for
all $x \in \operatorname{supp}(r_X)$. This yields
\[
q_X^\theta(x)=e^{-c^\theta}\tilde p_X(x)
\qquad \text{for } x\in\operatorname{supp}(r_X),
\]
but places no constraint on $q_X^\theta(x)$ outside the support of $r_X$.

As a result, one can modify $q_X^\theta$ arbitrarily on $X \setminus
\operatorname{supp}(r_X)$ while maintaining normalization, without affecting the
value of $\mathcal L_{\mathrm{LD}}^\theta(r_X;\rho)$. Therefore, $\mathcal
L_{\mathrm{LD}}^\theta(r_X;\rho)=0$ does not imply $q_X^\theta=p_X$ on $X$, and
log-dispersion fails to be a divergence in this case.
Appendix~\ref{appendix:no_forward_kl} illustrates this phenomenon empirically.
\end{proof}

This situation arises naturally when $r_X$ corresponds to an empirical data
distribution, e.g., obtained from molecular dynamics simulations, which typically
cover only a low-dimensional manifold or a limited region of configuration space.
In this regime, LD alone enforces agreement between $q_X^\theta$ and $p_X$ only on
the data manifold, leaving the model unconstrained elsewhere and allowing
incorrect normalization or spurious probability mass outside the data support.
Furthermore, also for variational methods such as CMT, where off-policy training
on fixed reweighted buffers is performed, the divergence property does not hold.

\subsection{Consistency of Log-Dispersion Regularization}

We now show that combining the log-dispersion objective with a standard
data-based divergence resolves this issue and restores the correct optimum even
when the reference distribution lacks full support. We call this approach
log-dispersion regularization.

Let $\mathcal L_{\mathrm{data}}^\theta$ denote a data-based divergence whose
unique minimum is attained at the target distribution $p_X$, such as the forward
KL divergence, score matching, or flow matching, depending on the model class.
We consider the combined objective
\[
\mathcal{L}^\theta
=
\lambda_\mathrm{data}\,\mathcal L_{\mathrm{data}}^\theta
+
\lambda_\mathrm{LD}\,\mathcal L_{\mathrm{LD}}^\theta(r_X;\rho),
\qquad
\lambda_\mathrm{data},\lambda_\mathrm{LD} > 0.
\]

\begin{proposition}[Consistency of log-dispersion regularization]
\label{prop:hybrid_consistency}
Assume that $\mathcal L_{\mathrm{data}}^\theta \ge 0$ and that
\[
\mathcal L_{\mathrm{data}}^\theta = 0
\quad\Longleftrightarrow\quad
q_X^\theta = p_X.
\]
Then the combined objective $\mathcal L^\theta$ satisfies
\[
\mathcal L^\theta \ge 0,
\qquad
\mathcal L^\theta = 0
\quad\Longleftrightarrow\quad
q_X^\theta = p_X,
\]
independently of whether the reference distribution used in
$\mathcal L_{\mathrm{LD}}^\theta(r_X;\rho)$ has full support.
\end{proposition}

\begin{proof}
Nonnegativity follows immediately from nonnegativity of both terms and the
assumption $\lambda_\mathrm{data},\lambda_\mathrm{LD} > 0$.

If $\mathcal L^\theta = 0$, then necessarily
$\mathcal L_{\mathrm{data}}^\theta = 0$, which implies $q_X^\theta = p_X$ by
assumption. Conversely, if $q_X^\theta = p_X$, then
$\mathcal L_{\mathrm{data}}^\theta = 0$ and
$f^\theta(x)=\log\mathcal Z$ is constant, yielding
$\mathcal L_{\mathrm{LD}}^\theta(r_X;\rho) = 0$ for any reference distribution $r_X$.
Thus $\mathcal L^\theta = 0$.
\end{proof}

This result shows that log-dispersion can be safely employed as a regularization
term when evaluated off-policy on a fixed dataset. The data-based objective
ensures global correctness and proper normalization of the model distribution,
while LDR provides an auxiliary training signal that aligns the learned density
with the target energy landscape on the data manifold by reducing dispersion in
the log importance weights. As a result, the combined objective preserves the
true target distribution as its unique optimum, while benefiting from the
additional structure encoded in the target energies.

\subsection{Gradients of Log-Dispersion Objectives}
In this subsection, we specialize to the LDR-Lp objectives obtained by choosing
$\rho(u)=|u|^p$ with $p\ge 1$. We analyze their gradients with respect to the
model parameters~$\theta$, with particular focus on the cases $p=1$ and $p> 1$.
Our goal is to understand the behavior of these gradients at the optimum
$q_X^\theta = p_X$.

Recall that
\[
f^\theta(x) = -\log q_X^\theta(x) - \frac{E(x)}{k_\mathrm{B}T},
\qquad
c^\theta = \mathbb E_{r_X}[f^\theta(x)],
\]
and
\[
\mathcal L_{\mathrm{LD}}^{\theta\,(p)}
=
\mathbb E_{r_X}\!\left[\,\big|f^\theta(x)- c^\theta\big|^p\,\right].
\]

\paragraph{Gradient of the LDR-Lp objective}
Assuming sufficient regularity to interchange gradient and expectation, the
gradient of $\mathcal L_{\mathrm{LD}}^{\theta\,(p)}$ is given by
\begin{align}
\nabla_\theta \mathcal L_{\mathrm{LD}}^{\theta\,(p)}
&=
p\,\mathbb E_{r_X}\!\left[
\big|f^\theta(x)- c^\theta\big|^{p-1}
\operatorname{sign}\!\big(f^\theta(x)- c^\theta\big)
\left(\nabla_\theta f^\theta(x) - \mathbb E_{r_X}[\nabla_\theta f^\theta(x)]\right)
\right],
\label{eq:ldr_grad_general}
\end{align}
where $\operatorname{sign}(u) = u/|u|$ for $u \neq 0$, and
$\operatorname{sign}(0)$ denotes the subdifferential $[-1,1]$.

At the optimum $q_X^\theta=p_X$, we have
\[
f^\theta(x) = \log \mathcal Z
\quad\text{for all }x,
\]
and therefore $f^\theta(x)- c^\theta = 0$ identically.
The behavior of the gradient at this point depends critically on the choice of
$p$.

\paragraph{Case $p=1$:}
For $p=1$, the LD objective reduces to
\[
\mathcal L_{\mathrm{LD}}^{\theta\,(1)}
=
\mathbb E_{r_X}\!\left[\big|f^\theta(x)- c^\theta\big|\right],
\]
with gradient
\begin{align}
\nabla_\theta \mathcal L_{\mathrm{LD}}^{\theta\,(1)}
=
\mathbb E_{r_X}\!\left[
\operatorname{sign}\!\big(f^\theta(x)- c^\theta\big)
\left(\nabla_\theta f^\theta(x)-\mathbb E_{r_X}[\nabla_\theta f^\theta(x)]\right)
\right].
\label{eq:l1_grad}
\end{align}

At the exact optimum, $f^\theta(x)- c^\theta = 0$ for all $x$. However, the
subdifferential of the absolute value at zero is the interval $[-1,1]$, so the
gradient is not uniquely defined. In the vicinity of the optimum, when
$f^\theta(x)- c^\theta$ is small but nonzero, the gradient contributions
in~\eqref{eq:l1_grad} remain of constant magnitude, independent of how close
$q_X^\theta$ is to $p_X$. This contrasts with $p > 1$, where gradient
contributions are smoothly damped as the optimum is approached (see below).
Consequently, in stochastic optimization, where the model never exactly reaches
$q_X^\theta = p_X$ and residual fluctuations persist, the L1 log-dispersion
objective can induce persistent gradient noise near the optimum.

\paragraph{Case $p> 1$:}
For $p> 1$, the gradient in~\eqref{eq:ldr_grad_general} contains the factor
$\big|f^\theta(x)- c^\theta\big|^{p-1}$. At the optimum, where
$f^\theta(x)- c^\theta = 0$ identically, this factor vanishes pointwise,
since $p-1>0$. Therefore,
\[
\nabla_\theta \mathcal L_{\mathrm{LD}}^{\theta\,(p)} = 0
\quad\text{for all }p>1.
\]

As the model distribution approaches the optimum $q_X^\theta=p_X$, the
deviations $f^\theta(x)- c^\theta$ shrink, and the factor $\lvert
f^\theta(x)- c^\theta\rvert^{p-1}$ in the gradient increasingly damps
gradient contributions. At the optimum, this damping becomes exact, and the
gradients of the LD term vanish identically. Consequently, once the model has
reached the target distribution, the log-dispersion regularizer induces no
further parameter updates.

\paragraph{Implications}
Although the $p=1$ objective can be more robust to outliers due to its linear
penalty, it may introduce persistent gradient noise near the optimum. In
contrast, objectives with $p> 1$ combine vanishing gradients at convergence
with increasing sensitivity to large deviations of $f^\theta(x)$ from its mean.
This trade-off suggests a natural distinction between the cases $p=1$ and $p>
1$ when employing log-dispersion regularization in practice.

\section{Experimental Setup}
\subsection{Architecture}
\label{appendix:architecture}

Here, we summarize the normalizing flow architectures used in this work. We
further show the approximate inference time of each architecture in
Table~\ref{SI:tab_inference_speed}. To put the inference times into perspective,
we also included a continuous normalizing flow, evaluated by exactly calculating
the trace of the Jacobian using the codebase of
\citeauthor{vaitlPathGradientsFlow2025a} \yrcite{vaitlPathGradientsFlow2025a}.

\input{tables/inference_speeds.tex}

\paragraph{GMM}
To train on the GMM target density, we use a RealNVP architecture
\cite{dinhDensityEstimationUsing2017a} with affine coupling layers. The model
consists of $15$ coupling layers, each containing a conditioner network
implemented as a fully-connected neural network with weight normalization and
batch normalization. Each conditioner network has two hidden layers of $160$
units with $\tanh$ activation functions. The coupling layers alternate their
masking pattern to ensure all dimensions are transformed. The base distribution
is a standard Gaussian $\mathcal{N}(0, I_2)$. 

\paragraph{Molecular systems in internal coordinates}
The normalizing flow architecture follows established designs used in prior
studies
\cite{midgleyFlowAnnealedImportance2022a,schopmansTemperatureAnnealedBoltzmannGenerators2025,schopmansConditionalNormalizingFlows2024}.
Molecular conformations are represented in internal coordinates consisting of
bond lengths, bond angles, and dihedral angles.

The model comprises 8 pairs of neural spline coupling layers based on monotonic
rational-quadratic splines \cite{durkanNeuralSplineFlows2019a}. Each spline
operates on the interval $[0,1]$ and uses 8 bins. Within each pair of coupling
layers, a random binary mask determines which dimensions are transformed and
which are conditioned upon in the first layer, while the negated mask is applied
in the second layer. Dihedral angle dimensions are handled using circular spline
transformations \cite{rezendeNormalizingFlowsTori2020a} to account for their
periodic topology, and a fixed random periodic shift is applied after every
coupling layer. The networks producing the spline parameters are fully connected
neural networks with hidden layer sizes $[256,256,256,256,256]$ and ReLU
nonlinearities. To encode periodicity, each dihedral angle $\psi_i$ is
represented as $(\cos \psi_i, \sin \psi_i)$ when provided as input to the
parameter networks.

The base distribution of the flow is chosen as a uniform distribution on $[0,1]$
for dihedral angles, and a truncated Gaussian on $[0,1]$ for bond lengths and
bond angles, with mean $\mu = 0.5$ and standard deviation $\sigma = 0.1$.

Following \citet{schopmansTemperatureAnnealedBoltzmannGenerators2025}, all
internal coordinates are mapped to the $[0,1]$ domain required by the spline
transformations. Dihedral angles are rescaled by division by $2\pi$. Bond
lengths and angles are shifted and scaled according to $\eta_i^\prime = (\eta_i
- \eta_{i;\text{min}}) / \sigma + 0.5$, where $\eta_{i;\text{min}}$ is taken
from a minimum-energy configuration obtained via energy minimization. The
scaling parameter $\sigma$ is set to \SI{0.07}{\nano\meter} for bond lengths and
to \num{0.5730} for bond angles.

The molecular systems considered admit two enantiomeric configurations
corresponding to L- and R-chirality, whereas naturally occurring structures
predominantly exhibit L-chirality. To restrict generated samples to the L-chiral
manifold, the output ranges of the splines corresponding to the relevant
dihedral angles are constrained as described in
\citet{schopmansTemperatureAnnealedBoltzmannGenerators2025}. In addition,
certain atoms or functional groups are formally permutation invariant in the
force-field energy, but appear in a fixed ordering in molecular dynamics data.
Analogously to the chirality constraints, the spline transformations are
restricted such that only the permutations observed in the validation data can
be generated \cite{schopmansTemperatureAnnealedBoltzmannGenerators2025}.

To implement the normalizing flow models on internal coordinate representations,
we used the \emph{bgflow} \cite{noeBgflow2024} and \emph{nflows}
\cite{conordurkanNflowsNormalizingFlows2020} libraries.

\paragraph{Molecular systems in Cartesian coordinates}
For our experiments on Cartesian coordinates, we leverage TarFlow
\cite{zhaiNormalizingFlowsAre2025a}, a block-wise autoregressive normalizing flow
based on a transformer architecture with causal masking. We closely follow the
hyperparameters introduced by
\citeauthor{tanScalableEquilibriumSampling2025a}
\yrcite{tanScalableEquilibriumSampling2025a}. We use 4 transformation blocks,
each with 4 attention layers and 256 channels. Cartesian coordinates are
transformed block-wise, treating the three Cartesian coordinates of each atom as
one block.

\subsection{Variational Training Based on Annealing Paths}
\label{appendix:variational_training_annealing}

This section provides a short overview of the annealing-based variational
sampling approaches CMT \cite{klitzingLearningBoltzmannGenerators2025a} and TA-BG
\cite{schopmansTemperatureAnnealedBoltzmannGenerators2025}.

Let $p_X:\mathbb{R}^d\to\mathbb{R}^+$ be a probability density function known up
to its normalization constant $\mathcal{Z}$, i.e.,
\begin{equation*}
    p_X(x)=\frac{\tilde p_X(x)}{\mathcal{Z}},\quad\text{with}\quad\mathcal{Z}=\int_{\mathbb{R}^d}\tilde p_X(x) \dd x.
\end{equation*}
A common strategy to generate samples $x \sim p_X(x)$ is to approximate $p_X$
with a parameterized variational model
$q_X^\theta\in\mathcal{Q}_\theta\subset\mathcal{P}(\mathbb{R}^d)$ by minimizing
the reverse KL divergence,
\begin{equation}
\label{eq:annealing_paths:rev_kl_objective}
    \min_\theta \KL(q_X^\theta\|p_X).
\end{equation}
While this approach works well for simple target distributions, more complex
targets often lead to mode collapse, where some modes of $p_X$ are poorly
represented or entirely missed by the model.

\paragraph{Annealing paths}
One strategy to mitigate this issue is the use of variational annealing paths,
which interpolate between an initial model distribution $q_0$ and the target
distribution $p_X$ through a sequence of intermediate densities $(q_i)_{i=1}^K$,
with $q_K=p_X$. Recent work on constrained variational objectives by
\citet{klitzingLearningBoltzmannGenerators2025a} (CMT) establishes a connection
between a family of constrained variational objectives and common annealing
paths. We therefore consider the geometric-tempered annealing path
\begin{equation*}
    q_i\propto q_0^{1-\lambda_i} (p_X^{1/T_i})^{\lambda_i},\quad \lambda_i\in\mathbb{R}_{\geq0},\,\,T_i\in[1,\infty),\,\,i=1\dots K,
\end{equation*}
proposed by CMT \cite{klitzingLearningBoltzmannGenerators2025a}, which
generalizes the widely used geometric annealing path
\begin{equation*}
    q_i\propto q_0^{1-\lambda_i} p_X^{\lambda_i},\quad \lambda_i\in\mathbb{R}_{\geq0},\,\,i=1\dots K,
\end{equation*}
and the temperature
annealing path 
\begin{equation*}
    q_i\propto p_X^{1/T_i},\quad T_i\in[1,\infty),\,\,i=1\dots K,
\end{equation*}
used by TA-BG \cite{schopmansTemperatureAnnealedBoltzmannGenerators2025}.

\paragraph{Annealing schedules}
TA-BG first pre-trains the model at an elevated temperature using the reverse
KL objective. Since modes are more interconnected at high temperatures, this
helps avoid mode collapse. TA-BG subsequently anneals the model distribution,
starting from the high-temperature distribution $q_0$, using a manually chosen
geometric temperature annealing schedule.

In contrast, CMT \cite{klitzingLearningBoltzmannGenerators2025a} skips the
pre-training phase and directly starts with an uninformed prior distribution
$q_0$. \citet{klitzingLearningBoltzmannGenerators2025a} then derive a
geometric-tempered annealing path and its adaptive schedule by analytically
solving the constrained variational objective
\begin{equation*}
    q_{i+1} = \argmin_{q_X\in\mathcal{P}(\mathbb{R}^d)} \KL(q_X\|p_X)\quad\mathrm{s.t.}\quad\KL(q_X\|q_i)\leq\varepsilon_\mathrm{tr},\quad \mathcal{H}(q_i)-\mathcal{H}(q_X)\leq\varepsilon_\mathrm{ent},\quad\int q_X(x)\,\dd x=1
\end{equation*}
for general probability measures $q\in\mathcal{P}(\mathbb{R}^d)$. While this
approach can adaptively choose both $(\lambda_i)_i$ (trust-region constraint)
and $(T_i)_i$ (entropy constraint), we found it simpler to select the
temperature schedule manually, as in TA-BG, while still enforcing the trust-region
constraint. Using a manual temperature schedule thus yields the objective
\begin{equation*}
    q_{i+1} = \argmin_{q_X\in\mathcal{P}(\mathbb{R}^d)} \KL(q_X\|p_X^{1/T_{i+1}})
    \quad \mathrm{s.t.} \quad \KL(q_X\|q_i) \le \varepsilon_\mathrm{tr}, \quad \int q_X(x)\,\dd x = 1.
\end{equation*}

\paragraph{Algorithm}

\begin{algorithm}[H]
\caption{Annealing algorithm}
\label{alg:annealing_abstract}
\begin{algorithmic}

\REQUIRE Initial density $q_0$, target density $\tilde p_X$, divergence $D$,
approximation family $\mathcal{Q}_\theta$

\FOR{$i \gets 1, \dots, K$}
    \item Initialize buffer $\mathcal{B}_i$ with samples from current model $q_X^\theta$
  \item Prepare for next intermediate target $q_i$ (e.g., choose $\lambda_i$ adaptively, compute importance weights, \dots)
  \FOR{$k\gets1,\dots,L$}
        \item Update $q_X^\theta$ by performing gradient descent on $D(q_i,q_X^\theta)$ using the buffer $\mathcal{B}_i$
  \ENDFOR
\ENDFOR

\item \textbf{return} $q_X^\theta \approx p_X$

\end{algorithmic}
\end{algorithm}

Given an annealing path and schedule, the original objective in
\Cref{eq:annealing_paths:rev_kl_objective} breaks down into a sequence of
simpler variational objectives
\begin{equation*}
    \min_\theta D(q_i,q_X^\theta),\quad\quad i=1\dots K,
\end{equation*}
where $D$ denotes a statistical divergence. Both CMT and TA-BG employ the
importance-weighted forward Kullback-Leibler divergence.

Both variational approaches follow a nested optimization structure, consisting
of an outer and an inner loop. In each outer iteration, a buffer of samples is
generated and then used off-policy to update the model parameters during the
inner loop. We can thus readily apply LDR in this setting. An abstract version
of the annealing algorithm is provided in \Cref{alg:annealing_abstract}.

\subsection{Target Densities}
\label{appendix:target_densities}

\paragraph{GMM target}
We use a Gaussian mixture model (GMM) in $d=2$ dimensions, closely following the
system introduced by \citeauthor{vaitlFastUnifiedPath2023}
\yrcite{vaitlFastUnifiedPath2023}. The target distribution is defined as an
equally-weighted mixture of $2^d$ Gaussian components arranged on a regular
grid. The component means are positioned at the vertices of a hypercube with
coordinates $\mu_i \in \{-1, +1\}^d$, resulting in $2^2 = 4$ modes located at
$(\pm 1, \pm 1)$. Each component is an independent isotropic Gaussian with
standard deviation $\sigma = 0.5$.

\paragraph{Molecular systems}
All our experiments on molecular systems were performed at \SI{300}{\kelvin}. An
overview of the molecular systems investigated in this work, together with the
corresponding force-field parametrizations, is provided in
Table~\ref{SI:tab_molecular_systems}. All energy evaluations used for model
training were carried out with OpenMM version 8.0.0
\cite{eastmanOpenMM8Molecular2024} on the CPU platform, employing 18 parallel
workers.

In line with previous studies
\cite{midgleyFlowAnnealedImportance2022a,schopmansTemperatureAnnealedBoltzmannGenerators2025,klitzingLearningBoltzmannGenerators2025a},
we employ a regularized energy formulation to mitigate excessively large van der
Waals contributions arising from atomic overlaps:

\begin{equation}
\label{eq:energy_reg}
E_{\text{reg}}(E)=
\begin{cases} 
E, & \text{if } E \leq E_{\text{high}}, \\
\log(E - E_{\text{high}} + 1) + E_{\text{high}}, & \text{if } E_{\text{high}} < E \leq E_{\text{max}}, \\
\log(E_\text{max} - E_{\text{high}} + 1) + E_{\text{high}}, & \text{if } E > E_{\text{max}}\,.
\end{cases}
\end{equation}

We choose $E_\text{high}=\num{1e8}$ and $E_\text{max}=\num{1e20}$, following the
values reported by \citeauthor{midgleyFlowAnnealedImportance2022a}
\yrcite{midgleyFlowAnnealedImportance2022a}.

\begin{table}[tp]
\caption{Summary of the molecular systems considered in this study, along with
the associated force-field parametrizations. To be consistent with prior work,
we consider two different force-field variants for alanine dipeptide. The first
is used for our experiments in internal coordinates, while the latter is used
for the experiments in Cartesian coordinates.}
\label{SI:tab_molecular_systems}
\begin{center}
\begin{small}
\begin{sc}
\begin{tabular}{ccccc}
\toprule
Name & No. atoms & Sequence & Force field & Constraints \\
\midrule
\multirow{3}{*}{\shortstack[c]{alanine \\ dipeptide}} & \multirow{3}{*}{22} & \multirow{3}{*}{ACE-ALA-NME} & \multirow{3}{*}{\normalfont \shortstack[c]{Amber ff96 \\ with OBC1 \\ implicit solvation}} & \multirow{3}{*}{None} \\
\\
\\
\multirow{3}{*}{\shortstack[c]{alanine \\ dipeptide}} & \multirow{3}{*}{22} & \multirow{3}{*}{ACE-ALA-NME} & \multirow{3}{*}{\normalfont \shortstack[c]{Amber ff99SB-ILDN \\ with Amber99 OBC \\ implicit solvation}} & \multirow{3}{*}{None} \\
\\
\\
\multirow{3}{*}{\shortstack[c]{alanine \\ hexapeptide}}  & \multirow{3}{*}{62} & \multirow{3}{*}{ACE-5$\cdot$ALA-NME} & \multirow{3}{*}{\normalfont \shortstack[c]{Amber ff99SB-ILDN \\ with Amber99 OBC \\ implicit solvation}} & \multirow{3}{*}{\shortstack[c]{Hydrogen \\ bond lengths}} \\
\\
\\
\bottomrule
\end{tabular}
\end{sc}
\end{small}
\end{center}
\vskip -0.1in
\end{table}

\input{tables/params_architecture_systems}

\subsection{Datasets}
\label{appendix:datasets}

\paragraph{GMM} For the GMM experiments, we used a test set containing
\num{10000} samples from the target distribution.

\paragraph{Molecular systems}
We performed extensive molecular dynamics simulations to obtain high-quality
ground truth datasets for both training and evaluation. Our simulation protocol
is similar to that reported by
\citeauthor{schopmansTemperatureAnnealedBoltzmannGenerators2025}
\yrcite{schopmansTemperatureAnnealedBoltzmannGenerators2025} and
\citeauthor{klitzingLearningBoltzmannGenerators2025a}
\yrcite{klitzingLearningBoltzmannGenerators2025a}.

For all molecular systems, we performed two independent simulations: the first
was used to build a training dataset with \num{5e6} samples and a validation
dataset of \num{1e6} samples, the second to build a test dataset with \num{1e7}
samples. The validation dataset was used to optimize hyperparameters, and the
test dataset to report the final metrics.

\paragraph{Alanine dipeptide} For alanine dipeptide, we performed molecular
dynamics simulations with the OpenMM integrator \emph{LangevinMiddle} and a time
step of \SI{1}{\femto \second}. We first equilibrated for \SI{200}{\nano
\second}, followed by the production simulation of \SI{5}{\micro \second}.

\paragraph{Alanine hexapeptide} For alanine hexapeptide, we performed
replica-exchange molecular dynamics
\cite{sugitaReplicaexchangeMolecularDynamics1999a} simulations at 
temperatures $[ \SI{300.0}{\kelvin}, \SI{332.27}{\kelvin}, \SI{368.01}{\kelvin},
\SI{407.60}{\kelvin}, \SI{451.44}{\kelvin}, \SI{500.0}{\kelvin} ]$ with the
OpenMM integrator \emph{LangevinMiddle} and a time step of \SI{1}{\femto
\second}. We first equilibrated each replica independently for \SI{200}{\nano
\second} without exchanges, then equilibrated for \SI{200}{\nano \second} with
exchanges, and subsequently performed the production simulation of \SI{2}{\micro
\second} per replica. To build our datasets, the data of the \SI{300.0}{\kelvin}
replica was used.

\paragraph{Biased dataset for alanine dipeptide}
To produce the biased dataset for alanine dipeptide (see
Section~\ref{sec:training_biased_datasets} of the main text), we used 4 starting
configurations located in the main minima of the energy surface of alanine
dipeptide, see Figure~\ref{fig:appendix_biased_aldp_dataset}. In each
configuration, we started 20 MD trajectories at \SI{300}{\kelvin}, each with a
length of $13750$ steps and time step \SI{1}{\femto \second} using a
\emph{LangevinMiddle} integrator. We discarded the first $1250$ steps of each
trajectory. Since we recorded every step of the trajectories, this resulted in a
biased dataset with \num{1e6} samples.

This construction uses prior knowledge of the main modes of alanine dipeptide
through the choice of starting configurations. Consequently, the total number of
target energy evaluations used in our biased experiments is not directly
comparable to the number of evaluations required for the unbiased
experiments in Section~\ref{sec:training_unbiased_datasets}; these costs should
therefore be interpreted separately.

\begin{figure*}[htbp]
  \vskip 0.2in
  \begin{center}
  \centerline{\includegraphics{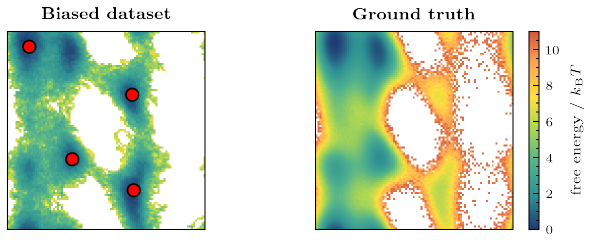}}
  \caption{Ramachandran plot of the biased training dataset for alanine
  dipeptide. The four starting configurations of the trajectories are labeled in
  red. Due to the short length of the trajectories, the high-energy metastable
  states on the right side are oversampled compared to the ground truth.}
  \label{fig:appendix_biased_aldp_dataset}
  \end{center}
\end{figure*}

\subsection{Metrics}
\label{appendix:section:metrics}

As discussed in the main part of our manuscript, we use the negative
log-likelihood (NLL) and ESS as the main metrics of our experiments. We also
introduce several additional metrics below. Metrics for the molecular
experiments in internal coordinates were estimated using \num{1e7} samples from
the model and ground truth (when needed). For the GMM experiments, we evaluated
the NLL on the separate test set of \num{10000} samples and used \num{100000}
model samples to estimate the reverse ESS. For the experiments in Cartesian
coordinates with TarFlow, we used \num{1e6} samples for evaluation due to the
reduced inference speed.

\paragraph{NLL} The NLL is defined as
\begin{align}
\mathrm{NLL}=-\mathbb{E}_{x \sim p_X(x)}\left[\log q_X^\theta(x)\right] \, .
\end{align}

For the experiments on Cartesian coordinates with data augmentation, we evaluate
the NLL on the augmented distribution $p_\text{aug}(x)$, as this is ultimately
the one learned by the model (see Appendix~\ref{appendix:theory:com_augm} for
details).

\paragraph{ESS} 
The ESS can be calculated using the following equation
\cite{martinoEffectiveSampleSize2017,midgleySE3EquivariantAugmented2023}:

\begin{align}
    &\text{ESS}=\frac{n_{\mathrm{e}, \mathrm{rv}}}{N}=\frac{1}{N \sum_{i=1}^N \bar{w}\left(x_i\right)^2} \label{eq:reverse_ESS} \\
    &\text{with} \quad x_i \sim q_X^\theta\left(x_i\right)\,\text{,} \quad \bar{w}(x_i) = \frac{w(x_i)}{\sum_{j=1}^N w(x_j)}\,\text{,} \quad w(x)=\frac{\tilde{p}_X(x)}{q_X^\theta(x)} \notag
\end{align}

For the molecular experiments in internal coordinates, the top
\SI{0.01}{\percent} of importance weights were truncated to their minimum value
within this subset when calculating the ESS, following previous work
\cite{midgleyFlowAnnealedImportance2022a,schopmansTemperatureAnnealedBoltzmannGenerators2025,klitzingLearningBoltzmannGenerators2025a}.
This clipping is performed to remove outliers due to model instabilities from
the importance weights. The effective sample size was computed using the
regularized energy function (Equation~\ref{eq:energy_reg}). For the GMM
experiments, we applied the same \SI{0.01}{\percent} clipping strategy. For the
Cartesian-coordinate experiments, following
\citet{tanScalableEquilibriumSampling2025a}, we instead discarded the samples
corresponding to the top \SI{0.2}{\percent} of importance weights entirely
before calculating the reverse ESS.

\paragraph{Ramachandran metrics}
Ramachandran plots visualize the two-dimensional log-density of the joint
distribution of backbone dihedral angle pairs $(\phi,\psi)$ in a peptide. They
capture the dominant conformational degrees of freedom of molecular systems and
are particularly sensitive to mode collapse and missing high-energy regions.

To quantitatively assess deviations between ground-truth and model-generated
Ramachandran plots, we follow previous work
\cite{midgleyFlowAnnealedImportance2022a,schopmansTemperatureAnnealedBoltzmannGenerators2025,klitzingLearningBoltzmannGenerators2025a}
and compute the forward Kullback-Leibler divergence between the discretized
ground truth Ramachandran distribution and the model distribution, which we call
\textbf{RAM KL}. Both distributions were estimated using $100 \times 100$ bins
over the full dihedral angle range. Since alanine hexapeptide has 5 pairs of
backbone dihedral angles, we average over the 5 corresponding \textbf{RAM KL}
metrics.

In addition, we report a reweighted variant of this metric (\textbf{RAM KL w.\
RW}), where model samples are first reweighted to the target distribution using
importance weights before constructing the Ramachandran histogram. To improve
the numerical stability of the reweighted estimate and to suppress the influence
of rare outliers caused by model instabilities, the same clipping procedure as
used for the ESS computation is applied: the top \SI{0.01}{\percent} of
importance weights are clipped to the minimum value within this subset. For the
Cartesian-coordinate experiments, we instead clipped the top \SI{0.2}{\percent}
of importance weights to the minimum value within this subset.

\paragraph{TICA metrics}
Time-lagged independent component analysis (TICA) provides a low-dimensional
representation of the slow collective degrees of freedom of molecular systems.
Distributions in TICA space are therefore sensitive to deficiencies in the
learned slow dynamics and to mode collapse in kinetically relevant regions.

To quantitatively assess deviations between ground-truth and model-generated
distributions in TICA space, we compute the forward Kullback--Leibler divergence
between the discretized ground truth TICA distribution and the model
distribution, which we call \textbf{TICA KL}. Both distributions were estimated
using $100 \times 100$ bins over the TICA ranges determined from the
ground-truth samples and using the first two TICA components.

In addition, we report a reweighted variant of this metric (\textbf{TICA KL w.\
RW}), where model samples are first reweighted to the target distribution using
importance weights before constructing the TICA histogram, using the same
importance weight clipping as used for \textbf{RAM KL w. RW}.

\paragraph{Metrics based on energy distribution}

In line with related work \cite{tanScalableEquilibriumSampling2025a}, we further
measure the discrepancy between generated and reference target energy
distributions by computing the 2-Wasserstein distance between their
one-dimensional distributions over potential energy values (estimated from
samples). Lower $\mathcal{E}$-$\mathcal{W}_2$ indicates a closer match of the
generated energy histogram to the reference. Since this metric is very sensitive
to outliers, in line with related work, we only report this metric after
categorical resampling of the flow sample distribution according to the
importance weights. For our experiments in internal coordinates, we use
\num{1e7} flow samples and resample to \num{1e7} samples. For our experiments in
Cartesian coordinates, we use \num{1e6} flow samples and resample to \num{1e6}
samples. We note that, in contrast to Wasserstein distances in $d>1$, they can
be efficiently calculated even for very large sample sizes in 1D.

We emphasize that this histogram energy metric compares only a one-dimensional
marginal distribution, and can therefore be less informative than metrics such
as ESS, which more directly reflects the quality of the reweighting and the
match to the target distribution. Moreover, across all our experiments the
reweighted energy histograms are already very close to the target; hence,
$\mathcal{E}$-$\mathcal{W}_2$ mainly captures small residual fluctuations, and
we caution against placing too much weight on small differences in this metric.

\paragraph{Further metrics used in related work}

Several related works report the torus Wasserstein distance
($\mathbb{T}\text{-}\mathcal{W}_2$) on backbone dihedral angles and the
Wasserstein distance in the TICA projection space, $\text{TICA-}\mathcal{W}_2$
\cite{tanScalableEquilibriumSampling2025a,tanAmortizedSamplingTransferable2025b,klitzingLearningBoltzmannGenerators2025a,rehmanFALCONFewstepAccurate2025b,rehmanEfficientRegressionBasedTraining2025}.

This metric is typically estimated using $10^4$ samples (some more recent works
use up to $2.5\times 10^5$ samples), where the sample-based dihedral angle
distribution can still be a very coarse proxy for the underlying target
distribution; for example, even in the two-dimensional Ramachandran marginals of
alanine dipeptide, the qualitative appearance changes drastically as the number
of reference samples increases from $10^4$ to $10^7$
(Figure~\ref{fig:appendix_aldp_10k_vs_1e7}). Especially to properly resolve the
high-energy metastable region on the right side of the Ramachandran, which is
crucial to be sensitive with respect to mode collapse, a large number of samples
is necessary. Increasing the number of samples to match the $10^7$ samples we
use for the metrics \textbf{RAM KL} and \textbf{TICA KL} would make
Wasserstein-based evaluation prohibitively expensive, particularly when
considering higher-dimensional torus products for longer peptides such as
alanine hexapeptide.

For these reasons, we focus on our KL-based metrics, which can be estimated
reliably at very large sample sizes and which we found to yield more stable and
reliable comparisons.

\begin{figure*}[htbp]
  \vskip 0.2in
  \begin{center}
  \centerline{\includegraphics{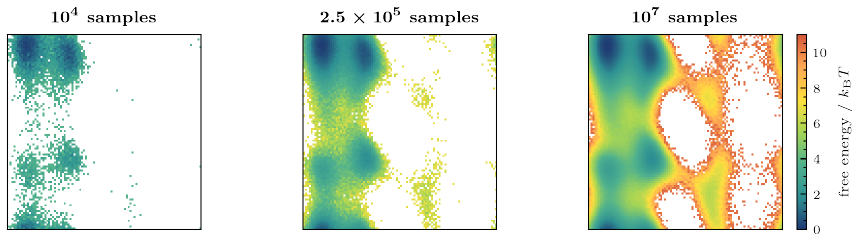}} 
  \caption{Ramachandran plots for alanine dipeptide, using randomly chosen
  \num{1e4} samples (left), \num{2.5e5} (middle), and \num{1e7} samples (right)
  from the ground truth dataset.}
  \label{fig:appendix_aldp_10k_vs_1e7}
  \end{center}
\end{figure*}

\subsection{Hyperparameters}
\paragraph{General} 
Unless stated otherwise below, experiments were performed using the Adam
optimizer~\cite{kingmaAdamMethodStochastic2017} with weight decay \num{1e-5},
and were implemented in \emph{PyTorch} \cite{paszkePyTorchImperativeStyle2019}.
All experiments included a cosine annealing learning rate scheduler with a
single cycle.

Initial learning rates were tuned separately for each experiment, using a
logarithmically spaced grid search. We first evaluated a coarse grid spanning
several orders of magnitude, and subsequently refined the search around the
best-performing region using a finer grid of the form $\{\ldots, 10^{-4},
3.2\times10^{-5}, 10^{-5}, \ldots\}$.

All molecular data-based training methods (unbiased and biased) used a batch
size of $1024$. For the unbiased GMM experiments, the batch size was always one
tenth of the training-dataset size: $50$, $100$, and $1000$ for datasets of
$500$, $1000$, and $10000$ samples, respectively. The GMM experiments did not
use weight decay.

The total number of gradient descent steps for each method can be found in
Table~\ref{SI:tab_training_times}.

\paragraph{LDR}
To report best-case performance with LDR, we tuned $\lambda_{\text{data}}$ over
the grid $\{0.1, 0.3, 0.5, 0.7, \ldots\}$ while fixing $\lambda_{\text{LD}} = 1$
for the experiments in internal coordinates.  We emphasize, however, that LDR
performs well across a range of choices of $\lambda_\text{data}$, as shown in
Figures~\ref{fig:appendix_loss_weight_aldp}
and~\ref{fig:appendix_loss_weight_hexa}. For the experiments in Cartesian
coordinates, we set $\lambda_{\text{data}}=1$ and tuned $\lambda_{\text{LD}}$
over the grid $\{0.1, 0.3, 0.5, 0.7, \ldots\}$.

In the data-based LDR experiments, we additionally warmed up the effective
regularization weight at the beginning of training. This makes the early
optimization steps dominated by the data-based objective before LDR is gradually
switched on. Specifically, we multiplied the final value of
$\lambda_{\text{LD}}$ by a cosine ramp,
\begin{align}
    \lambda_{\text{LD}}(t)
    =
    \lambda_{\text{LD}}
    \begin{cases}
        \frac{1}{2}\!\left(1-\cos\!\left(\pi t / T_{\text{warm}}\right)\right),
        & t < T_{\text{warm}},\\
        1, & t \geq T_{\text{warm}}.
    \end{cases}
\end{align}
For the unbiased LDR experiments in internal coordinates, we used
$T_{\text{warm}}=\num{2e4}$ optimization steps. For the biased LDR experiments
on alanine dipeptide, we used $T_{\text{warm}}=\num{5e4}$ optimization steps.
For the unbiased LDR experiments in Cartesian coordinates, we used
$T_{\text{warm}}=\num{1e5}$ optimization steps. For CMT + LDR, we did not
schedule $\lambda_{\text{LD}}$ and instead used the final weight throughout
training.

\paragraph{Biased experiments}
For the importance sampling step in the biased experiments, we categorically
resampled the importance-weighted dataset to a new dataset of \num{1e7} samples.

\paragraph{CMT}
We used a batch size of $1000$ for alanine dipeptide and $2000$ for alanine
hexapeptide. Gradient norm clipping was applied with a threshold of $100$. At
the start of training, the learning rate was linearly warmed up over the first
$1000$ steps. The trust-region bound was set to $0.3$ for all experiments.

For alanine dipeptide, we used $200$ outer annealing steps, while $400$
annealing steps were used for alanine hexapeptide. To compute the number of
samples for the buffer in each annealing step, we uniformly distributed the
total number of target energy evaluations over the number of annealing steps.

As described in Section~\ref{appendix:variational_training_annealing}, we did
not use the entropy constraint introduced by
\citeauthor{klitzingLearningBoltzmannGenerators2025a}
\yrcite{klitzingLearningBoltzmannGenerators2025a}. Instead, we chose a manual
geometric temperature schedule that geometrically anneals the temperature from
\SI{1200}{\kelvin} to \SI{300}{\kelvin} over the first half of the total number
of gradient descent steps. With this fixed schedule, we observed CMT to achieve
comparable results as when using an entropy constraint. In this setup, CMT is
very similar to TA-BG with an additional trust-region constraint to improve
overlap of consecutive distributions.

\paragraph{TA-BG}
For our experiments with TA-BG, we started from the exact hyperparameters
specified by \citeauthor{schopmansTemperatureAnnealedBoltzmannGenerators2025}
\yrcite{schopmansTemperatureAnnealedBoltzmannGenerators2025}. We refer to
\cite{schopmansTemperatureAnnealedBoltzmannGenerators2025} for details. Reverse
KL pre-training was performed at \SI{1200}{\kelvin}, from where a geometric
annealing path to \SI{300}{\kelvin} was followed. To ensure a fair comparison,
we matched the number of gradient descent steps used in TA-BG to those used in
CMT (see Table~\ref{SI:tab_training_times}).

\paragraph{FAB}
For our experiments with FAB, we used the exact hyperparameters specified by
\citeauthor{klitzingLearningBoltzmannGenerators2025a}
\yrcite{klitzingLearningBoltzmannGenerators2025a}, without additional tuning.

\paragraph{Cartesian-coordinate experiments}
For the experiments with TarFlow, we used AdamW with
$(\beta_1,\beta_2)=(0.9,0.95)$, weight decay \num{4e-4}, and gradient norm
clipping at $1.0$. These choices were inspired by the training setup of
\citet{tanScalableEquilibriumSampling2025a}.

\subsection{Training Times}
We summarize the total training time (excluding evaluation) observed on an
NVIDIA A100 GPU (40\,GB memory) for each method in
Table~\ref{SI:tab_training_times}. We emphasize that LDR does not change the
required training time and can be added with no extra cost.

\input{tables/training_times.tex}

\section{Training on Cartesian Coordinates}
\subsection{Augmentation}
\label{appendix:theory:com_augm}

Directly enforcing translational and rotational invariance of the target
distribution within a discrete normalizing flow architecture is challenging due
to inherent architectural constraints of normalizing flows. As a result,
\citeauthor{tanScalableEquilibriumSampling2025a}
\yrcite{tanScalableEquilibriumSampling2025a} adopt an alternative strategy and
train a non-invariant flow, while accounting for these symmetries through data
augmentation as follows:

\begin{enumerate}
\item Let $X=\mathbb{R}^{N\times 3}$. For a configuration $x\in X$, write
$x_i\in\mathbb{R}^3$ for its $i$-th row, i.e., the Cartesian position of
particle $i$, and define
\[
\bar x \coloneqq \frac{1}{N}\sum_{i=1}^N x_i,
\qquad
x^\circ \coloneqq x-\mathbf{1}_N \bar x^\top,
\]
where $\mathbf{1}_N\in\mathbb{R}^N$ denotes the column vector of ones. Thus
$\mathbf{1}_N \bar x^\top\in\mathbb{R}^{N\times 3}$ repeats the center of mass
in every row.

\item Define the augmentation map
\[ s: X_0 \times \mathbb{R}^3 \to X, \qquad s(x^\circ, t) = x^\circ +
\mathbf{1}_N t^\top. \]

\item Sample an independent translation:
\[
t \sim \mathcal{N}(0,\sigma_t^2 I_3).
\]

\item Apply the augmentation (push-forward):
\[
x_{\mathrm{aug}} = s(x^\circ, t)
= x^\circ + \mathbf{1}_N t^\top
\]
\textbf{Note:} While molecular dynamics (MD) simulations, in principle, explore
all possible molecular rotations due to the rotation-invariance of $p_X$, in
practice, we additionally apply randomly sampled rotations during training,
which increases sample diversity and robustness.
\end{enumerate}

Adding noise to the center of mass is crucial. Without this noise, the data
distribution lies on a \(3N-3\)-dimensional submanifold of \(\mathbb{R}^{3N}\).
As a result, any mapping from a full-dimensional base distribution would require
a singular Jacobian, violating the invertibility assumptions of normalizing
flows.

In the following, we derive the push-forward density $p_{\text {aug }}(x)$ of
the augmentation scheme described above.

\begin{lemma}[Center-of-mass augmentation]
Let $p_{X_0}$ be a probability density on the centered subspace 
\[ X_0 \coloneqq \{z \in X : \bar z = 0\}, \]
assume $p_{X_0}$ is rotation invariant, i.e., $p_{X_0}(zR)=p_{X_0}(z)$ for all $R\in\mathrm{SO}(3)$,
and let $t \sim \mathcal{N}(0, \sigma_t^2 I_3)$.

The augmented density $p_{\mathrm{aug}}(x)$ on the full
space $X$, given by the push-forward using the map $s$, is, up to normalization:
\[ p_{\mathrm{aug}}(x) \propto p_{X_0}(x^\circ) \, \mathcal{N}(\bar x \mid 0, \sigma_t^2 I_3). \]
\end{lemma}

\begin{proof}
We begin with the joint density of the independent variables $x^\circ \in X_0$, $t \in \mathbb{R}^3$:
\[ p(x^\circ, t) = p_{X_0}(x^\circ) \, \mathcal{N}(t \mid 0, \sigma_t^2 I_3). \]

A change-of-variables of this joint distribution using the map $s(x^\circ,t)=x$ leads
to an augmented distribution on the full space, given by

\[ p_{\mathrm{aug}}(x) = p\big(s^{-1}(x)\big) \cdot \left| \det J \right|^{-1} \quad\text{with}\quad J=\frac{\partial s}{\partial(x^\circ, t)} \]

Substituting the joint density expression and the transformation, we obtain:
\[ p_{\mathrm{aug}}(x) \propto p_{X_0}\big(x - \mathbf{1}_N\bar{x}^\top\big) \, \mathcal{N}(\bar{x} \mid 0, \sigma_t^2 I_3) \cdot |\det J|^{-1} \]

Since $s$ is a linear transformation, its Jacobian determinant is constant. This leaves us with
\[ p_{\mathrm{aug}}(x) \propto p_{X_0}(x^\circ) \, \mathcal{N}(\bar{x} \mid 0, \sigma_t^2 I_3), \]
which completes the proof.
\end{proof}

\paragraph{Augmented loss functions}
We can now express the augmented versions of both the forward KL and the LD
objective using the non-augmented target distribution $p_{X_0}(x)$ as
\begin{equation*}
    \min_\theta\KL(p_\mathrm{aug}\|q_X^\theta)=\min_\theta -\mathbb{E}_{p_\mathrm{aug}(x)}[\log q_X^\theta(x)]=\min_\theta \underset{\substack{x^\circ \sim p_{X_0} \\ t \sim \mathcal{N}(0,\sigma_t^2I_3)}}{-\mathbb{E}}[\log q_X^\theta(x^\circ+\mathbf{1}_N t^\top)]
\end{equation*}
and
\begin{align*}
    \min_\theta\mathcal{L}_\mathrm{LD}^{\theta}
    &= \min_\theta
    \underset{p_\mathrm{aug}(x)}{\mathbb{E}}\!\left[
    \rho\!\left(f^\theta(x)
    -\underset{p_\mathrm{aug}(x)}{\mathbb{E}}[f^\theta(x)]\right)
    \right] \\
    &= \min_\theta
    \underset{\substack{x^\circ \sim p_{X_0} \\ t \sim \mathcal{N}(0,\sigma_t^2I_3)}}{\mathbb{E}}\!\left[
    \rho\!\left(f^\theta(x^\circ, t)
    -\underset{\substack{x^\circ \sim p_{X_0} \\ t \sim \mathcal{N}(0,\sigma_t^2I_3)}}{\mathbb{E}}[f^\theta(x^\circ, t)]\right)
    \right],
\end{align*}
with
\begin{equation*}
    f^\theta(x^\circ, t)=\log \tilde p_{X_0}(x^\circ)+\log\mathcal{N}(t|0,\sigma_t^2I_3)-\log q_X^\theta(x^\circ+\mathbf{1}_N t^\top).
\end{equation*}

\paragraph{Importance sampling correction}
Using the loss functions above, the normalizing flow learns to model the
augmented distribution $p_{\text{aug}}(x)$, rather than the (non-augmented)
target density $p_{X_0}$. To account for this during importance sampling with
respect to $p_{X_0}$,
\citeauthor{tanScalableEquilibriumSampling2025a}
\yrcite{tanScalableEquilibriumSampling2025a} introduced the following correction
of the proposal density:

\begin{align} 
\log q_X^{\theta\, c}(x)=\log q_X^\theta(x)+\frac{\|\bar x\|^2}{2 \sigma_t^2}-\log \left[\frac{\|\bar x\|^2}{\sqrt{2} \sigma_t^3 \Gamma\left(\frac{3}{2}\right)}\right]
\end{align}

This correction effectively removes a $\chi_3$ distribution from the proposal
density.

However, as shown above, the augmented density that the normalizing flow learns
factorizes as 

\begin{align}
p_{\mathrm{aug}}(x) \propto
p_{X_0}(x^\circ) \, \mathcal{N}(\bar{x} \mid 0, \sigma_t^2 I_3) \, .
\end{align}

Consequently, the discrepancy between the learned proposal and the target
density arises solely from the additional Gaussian factor in the augmented
coordinates. A natural correction, therefore, consists of explicitly removing this
Gaussian contribution from the proposal density, yielding

\begin{align}
\log q_X^{\theta\,c}(x)
&= \log q_X^\theta(x)
- \log \mathcal{N}(\bar x \mid 0, \sigma_t^2 I_3)
\end{align}

This recovers a proposal density consistent with importance sampling toward
$p_{X_0}$. Our results show empirically that this correction yields improved
results compared to the one proposed by
\citeauthor{tanScalableEquilibriumSampling2025a}
\yrcite{tanScalableEquilibriumSampling2025a} (Table~\ref{tab:results_tarflow}
and Table~\ref{tab:results_tarflow_all}).

\subsection{Extended Discussion of Results}
\label{appendix:extended_discussion_cart} 

As mentioned in the main text, we also used the path-gradient estimator for the
forward KL objective as a baseline in our Cartesian-coordinate experiments (see
Table~\ref{tab:results_tarflow_all}); we provide additional details here. For
this method, we relied on the ``plug-and-play'' implementation provided by
\citeauthor{vaitlFastUnifiedPath2023} \yrcite{vaitlFastUnifiedPath2023}. In our
setup, this implementation was approximately $4\times$ slower per optimizer step
than the other training objectives considered. To match the overall
computational budget, we therefore ran path-gradient forward KL for only
\num{100000} optimizer steps, whereas the other methods were trained for
\num{400000} steps.

Despite this constraint, we do observe the characteristic low-variance behavior
of path-gradient forward KL: when measured as a function of the number of
gradient updates, it tends to converge faster than the alternatives. However,
this advantage does not translate to improved performance as a function of
wall-clock time in our experiments, due to its substantially higher per-step
cost.

We note that \citet{vaitlPathGradientsFlow2025a} addressed a closely related
issue in the context of continuous normalizing flows---where the computational
gap is much more pronounced---by first pre-training with flow matching and then
fine-tuning with path gradients. However, to keep the experimental setup simple,
we did not include pre-training and fine-tuning pipelines in our setting.

Finally, we emphasize that path-gradient forward KL is not a direct competitor
to LDR. Conceptually, path-gradient forward KL regularizes training through the
inclusion of target \emph{gradients}, whereas LDR regularizes through the
inclusion of target \emph{energy values}. These mechanisms are complementary and
can be readily combined in future work.

\section{Additional Experiments and Discussion}
\subsection{Sensitivity to Loss Weights}
\label{appendix:loss_weight_sensitivity}

\begin{figure*}[htbp]
  \vskip 0.2in
  \begin{center}
  \centerline{\includegraphics{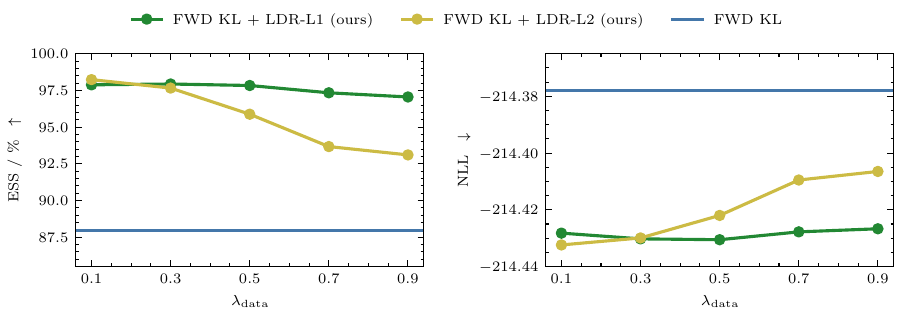}}
  \caption{Final ESS and NLL as a function of the loss weight
  $\lambda_\text{data}$ for unbiased training on alanine dipeptide using
  \num{1e6} samples.}
  \label{fig:appendix_loss_weight_aldp}
  \end{center}
\end{figure*}

\begin{figure*}[htbp]
  \vskip 0.2in
  \begin{center}
  \centerline{\includegraphics{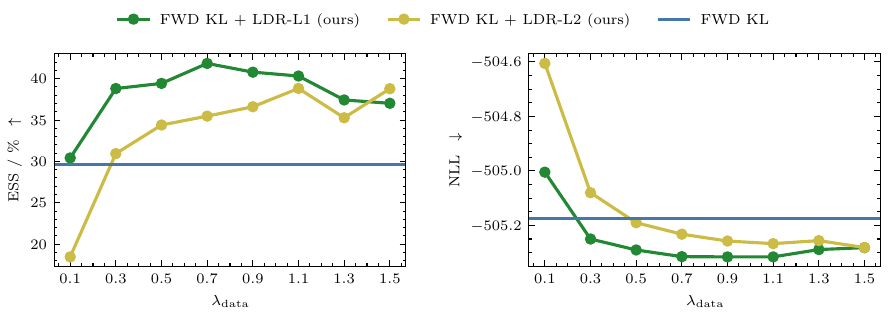}}
  \caption{Final ESS and NLL as a function of the loss weight
  $\lambda_\text{data}$ for unbiased training on alanine hexapeptide using
  \num{5e6} samples.}
  \label{fig:appendix_loss_weight_hexa}
  \end{center}
\end{figure*}

Compared to standard data-based training, our method introduces an additional
hyperparameter that scales the two loss components relative to each other. We
visualize the sensitivity of the final model performance to the choice of the
hyperparameter $\lambda_\text{data}$ in
Figures~\ref{fig:appendix_loss_weight_aldp}
and~\ref{fig:appendix_loss_weight_hexa}.
For each chosen $\lambda_\text{data}$, we tune the learning rate separately to
remove confounding effects caused by non-optimal settings. As one can see, both
LDR-L1 and LDR-L2 outperform the baseline without regularization across a large
range of chosen $\lambda_\text{data}$. Furthermore, LDR-L1 appears somewhat more
stable with respect to $\lambda_\text{data}$ compared to LDR-L2.

\subsection{Training without a Data-Based Objective}
\label{appendix:no_forward_kl}

As outlined in Appendix~\ref{appendix:theory}, LDR needs to be combined with a
divergence that ensures correct normalization of the proposal.
Figure~\ref{fig:appendix_only_ldr} illustrates what happens when only LDR is
used for training on an unbiased dataset. While the proposal matches the target
where the dataset has support, it is unconstrained outside of the data manifold.

The same problem occurs when running CMT (as well as TA-BG) with only LDR. Both
methods build buffers by reweighting model samples to an intermediate target
distribution. Since the LDR objective is evaluated over this reweighted fixed
dataset while training on each buffer, the same problem exists; the objective
does not constrain the model density outside the data manifold.

\begin{figure*}[htbp]
  \vskip 0.2in
  \begin{center}
  \centerline{\includegraphics{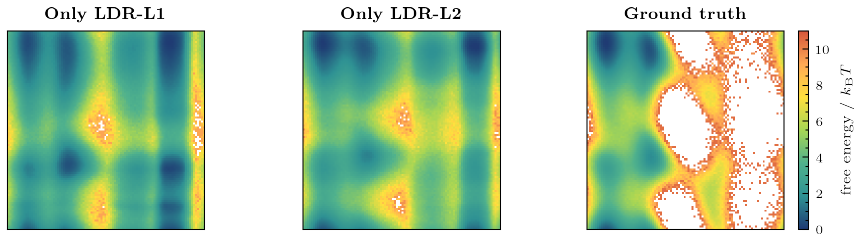}}
  \caption{Ramachandran obtained for alanine dipeptide when training with only
  LDR-L1 or only LDR-L2, without an additional data-based divergence.}
  \label{fig:appendix_only_ldr}
  \end{center}
\end{figure*}

\subsection{Additional Log-Dispersion Variants}
\label{sec:additional_LD}

In the main text of this work, we focused on the LDR-L1 and LDR-L2 variants from
the log-dispersion family. In Table~\ref{tab:results_additional_ld_variants}, we
compare the performance of additional variants when training on unbiased data
for alanine hexapeptide (\num{1e6} samples). This additionally includes LDR-L1.5
($p=1.5$), LDR-L3 ($p=3$), and LDR-Huber. For LDR-Huber, we use
\[
\rho(u)=
\begin{cases}
u^2, & |u|\le \delta,\\
2\delta |u|-\delta^2, & |u|>\delta,
\end{cases}
\]
with $\delta=1.0$ in our experiments.

\input{tables/results_additional_ld_variants.tex}

Table~\ref{tab:results_additional_ld_variants} shows that LDR-L1, LDR-L1.5, and
LDR-Huber achieve similar performance on this task. In contrast, LDR-L2 performs
worse than variants with weaker tail sensitivity, and LDR-L3 degrades performance
further. This supports the interpretation from the main text that penalties with
weaker sensitivity to large residuals can provide more stable optimization in the
heavy-tailed importance-weight regime of alanine hexapeptide.

This trend is also visible in Figure~\ref{fig:ld_variants_grad_norms}, where we
plot the variance of the L2 gradient norm inside a sliding window during
training. LDR-L1 exhibits the lowest gradient-norm variance, while LDR-L2 and
LDR-L3 show larger variation, indicating less stable training dynamics.

\begin{figure*}[htbp]
  \vskip 0.2in
  \begin{center}
  \centerline{\includegraphics{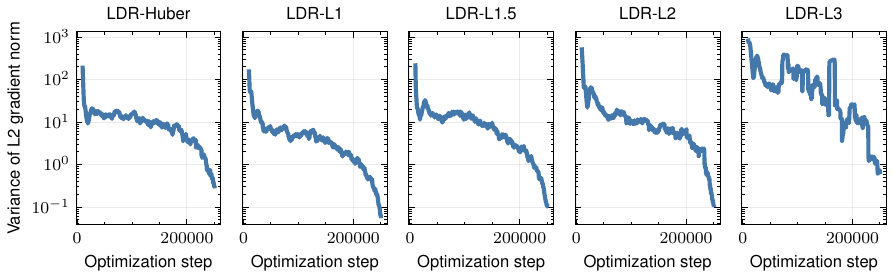}}
  \caption{Variance of L2 gradient norm inside a sliding window of \num{10000}
  optimization steps during training. We compare multiple variants from the
  log-dispersion family for training on unbiased data for alanine hexapeptide
  with \num{1e6} samples.}
  \label{fig:ld_variants_grad_norms}
  \end{center}
\end{figure*}

\subsection{Pushing the Limits of CMT} 
\label{sec:appendix:pushing_the_limits}
 
The LDR objective introduces an additional training signal, enabling effective
variational training with substantially fewer gradient descent steps and total
target energy evaluations. While our main experiments use \num{1e7} target
evaluations as the lowest setting of CMT for alanine dipeptide, this ablation
further reduces the evaluation budget.

\Cref{fig:cmt_fast} reports the final performance of CMT, CMT + LDR-L1, and CMT
+ LDR-L2 under a linear scaling of both the number of gradient descent steps and
the buffer size per annealing step, while fixing the number of annealing steps
at 200. Although performance degrades as the training budget is reduced, CMT
with log-dispersion regularization is significantly more robust to fewer
gradient descent steps and target energy evaluations than non-regularized CMT.
In particular, CMT + LDR-L1 reaches a final ESS of $82.98\,\%$ in approximately
$130$ minutes of wall-clock time on an NVIDIA A100 GPU (40\,GB memory) using
only $10^6$ target energy evaluations. Non-regularized CMT yields very poor
performance in this setting.

\begin{figure*}[htbp]
  \vskip 0.2in
  \begin{center}
  \centerline{\includegraphics{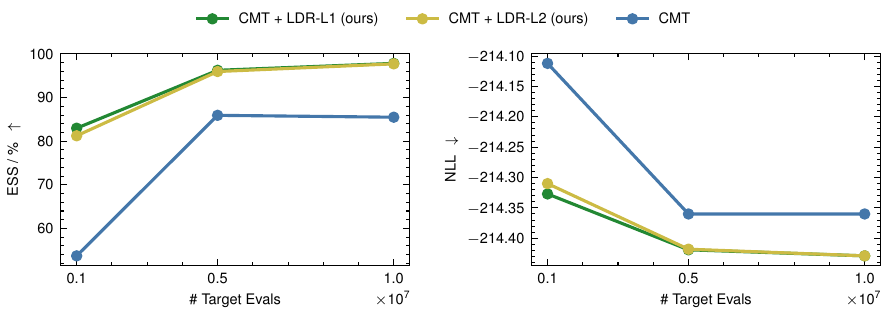}}
  \caption{Final ESS and NLL for a total of $10^6$, $5\times10^6$, and $10^7$ energy evaluations.}
  \label{fig:cmt_fast}
  \end{center}
\end{figure*}

\subsection{Connection to the Path-Gradient Forward KL}
\label{appendix:path_gradients}

Originally, path gradients were developed as a method to obtain a lower-variance
gradient estimator of the reverse KL objective
\cite{vaitlGradientsShouldStay2022}. \citet{vaitlFastUnifiedPath2023} noted that
the forward KL in data space can be interpreted as a reverse KL in the latent
space. This allows one to apply the path-gradient idea also to data-based
forward KL training. The objective that results from this approach is a
data-based objective that includes target gradient information.
\citet{vaitlFastUnifiedPath2023} hypothesized that the inclusion of target
gradients yields regularized training compared to the standard forward KL and
showed this empirically.

Both LDR and the path-gradient forward KL can be interpreted as
variance-reduction techniques applied at different levels: path gradients reduce
gradient estimator variance, while LDR reduces dispersion in the importance
weights directly. To achieve this, LDR leverages target energies (zero-order
regularizer), while the path-gradient forward KL uses target gradients
(first-order regularizer).

LDR is a flexible regularization framework that can be combined with any
data-based objective, including the path-gradient forward KL. Combining both
yields a doubly regularized objective: path-gradient forward KL provides a
first-order regularizer (target gradients), while LDR adds a zero-order
regularizer (target energies). Empirical evaluation of this combined approach is
an interesting direction for future work.

\paragraph{Comparison of the regularizing effect}
We directly compare the regularizing effect of forward KL + LDR with the path
gradient forward KL for the Gaussian mixture system and on alanine dipeptide
using Cartesian coordinates (see Table~\ref{tab:results_unbiased} and
Table~\ref{tab:results_tarflow_all}).

We did not observe stable training on the molecular tasks with internal
coordinates using neural spline flows. We tested several variants of gradient
norm clipping and tuned the learning rate extensively. We hypothesize that the
problem lies in the non-continuity of the gradients of the log-density of the
neural spline flow (due to the periodic handling of torsion angles), which makes
regularization using gradients of the target density difficult.

\clearpage
\section{Extended Results}
\label{appendix:extended_results}

\FloatBarrier

\input{tables/results_unbiased_all_GMM.tex} 
\input{tables/results_unbiased_all_molecular.tex}

\input{tables/results_biased_all.tex}

\input{tables/results_CMT_all.tex}

\input{tables/results_tarflow_all.tex}

\newpage
\section{Visualizations} 

\label{sec:appendix:visualizations}

\begin{figure*}[!htbp]
  \vskip 0.2in
  \begin{center}
  \centerline{\includegraphics{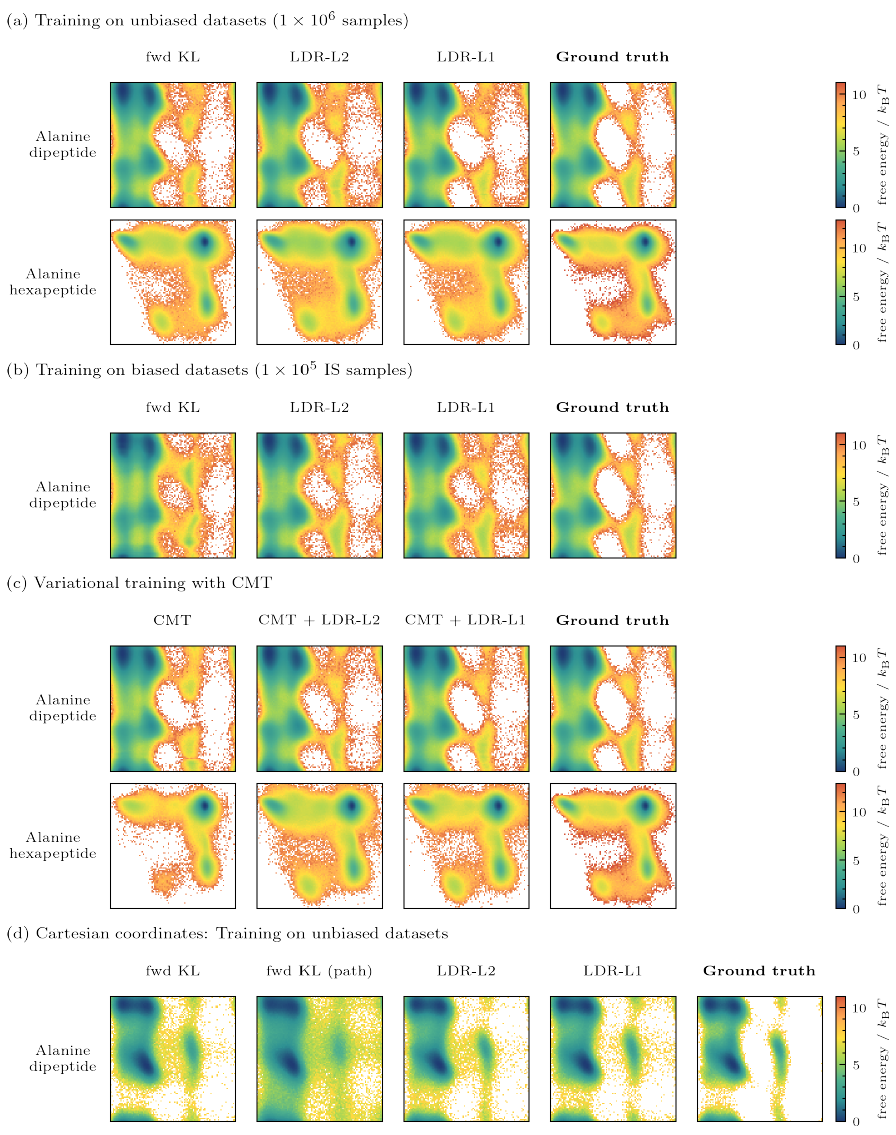}}
  \caption{Visualization of the 2D marginals of the main degrees of freedom for
  each system. For alanine dipeptide, we show the marginal of the two main
  dihedral angles (Ramachandran), for alanine hexapeptide the 2D TICA
  projection. (a) Training on unbiased datasets using \num{1e6} samples, (b)
  training on biased dataset using \num{1e5} IS samples, (c) variational
  training with CMT (\num{1e7} target evaluations for alanine dipeptide,
  \num{1e8} for alanine hexapeptide), (d) training in Cartesian coordinates on
  unbiased dataset (\num{1e5} samples).}
  \label{fig:all_ramachandran_tica_overview}
  \end{center}
\end{figure*}

\end{document}

%% file: tables/results_unbiased.tex
\begin{table*}
\caption{Results for training on unbiased datasets from the target distribution.
We report negative log-likelihood (NLL; lower is better) and effective sample size (ESS; higher is better), averaged over 4 independent experiments with standard deviations. 
For each number of samples, \textbf{bold} indicates metrics not significantly different from the best under 
uncorrected two-sided Welch's t-tests ($\alpha = 0.05$), capturing seed-level variability only. 
Right: ESS as a function of the number of training samples.}
\label{tab:results_unbiased}
\begin{center}
\begin{small}
\begin{sc}
\resizebox{\textwidth}{!}{\begin{tabular}{cclccc}
\toprule
System & {\#} Samples & Method & NLL $\downarrow$ & ESS / \% $\uparrow$ & \adjincludegraphics[valign=m]{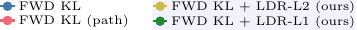} \\
\midrule
\multirow{3}{*}{GMM} & \multirow{4}{*}{\num{500}}  & fwd KL & \num{2.850(0.061)e0} & \num{77.76(6.90)} & \multirow{9}{*}{\adjincludegraphics[valign=m]{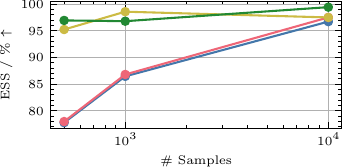}} \\
\multirow{8}{*}{\adjincludegraphics[valign=b,width=1.8cm]{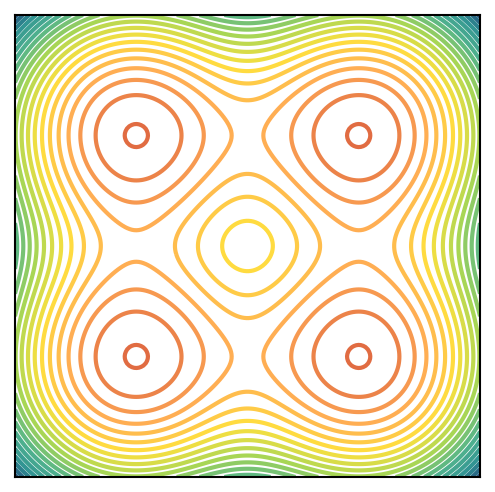}} & & fwd KL (path gradients) & \num{2.812(0.008)e0} & \num{77.97(5.32)} \\
                            &                                       & \cellcolor{blue!5}fwd KL + LDR-L2 (ours) & \cellcolor{blue!5}\textbf{\num{2.736(0.003)e0}} & \cellcolor{blue!5}\textbf{\num{95.26(2.48)}} \\
                            &                                       & \cellcolor{blue!5}fwd KL + LDR-L1 (ours) & \cellcolor{blue!5}\textbf{\num{2.734(0.004)e0}} & \cellcolor{blue!5}\textbf{\num{96.98(0.60)}} \\
\cmidrule(lr){2-5}
                            & \multirow{4}{*}{\num{1000}}  & fwd KL & \num{2.791(0.013)e0} & \num{86.47(1.43)} \\
                            &                                        & fwd KL (path gradients) & \num{2.777(0.027)e0} & \num{86.82(5.43)} \\
                            &                                        & \cellcolor{blue!5}fwd KL + LDR-L2 (ours) & \cellcolor{blue!5}\textbf{\num{2.725(0.001)e0}} & \cellcolor{blue!5}\textbf{\num{98.62(0.31)}} \\
                            &                                        & \cellcolor{blue!5}fwd KL + LDR-L1 (ours) & \cellcolor{blue!5}\textbf{\num{2.727(0.004)e0}} & \cellcolor{blue!5}\textbf{\num{96.83(2.60)}} \\
\addlinespace[0.1cm]
\midrule
\multirow{3}{*}{\shortstack[c]{Alanine \\ Dipeptide}} & \multirow{3}{*}{\shortstack[c]{\num{1e6}$\,^{\boldsymbol{\bigstar}}$}} & fwd KL & \num{-214.378(0.000)e0} & \num{87.94(0.06)} &
\multirow{7}{*}{\adjincludegraphics[valign=m]{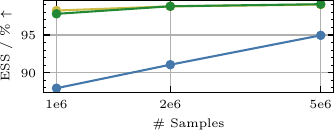}}
\\
\multirow{7}{*}{\adjincludegraphics[valign=b,width=2.7cm]{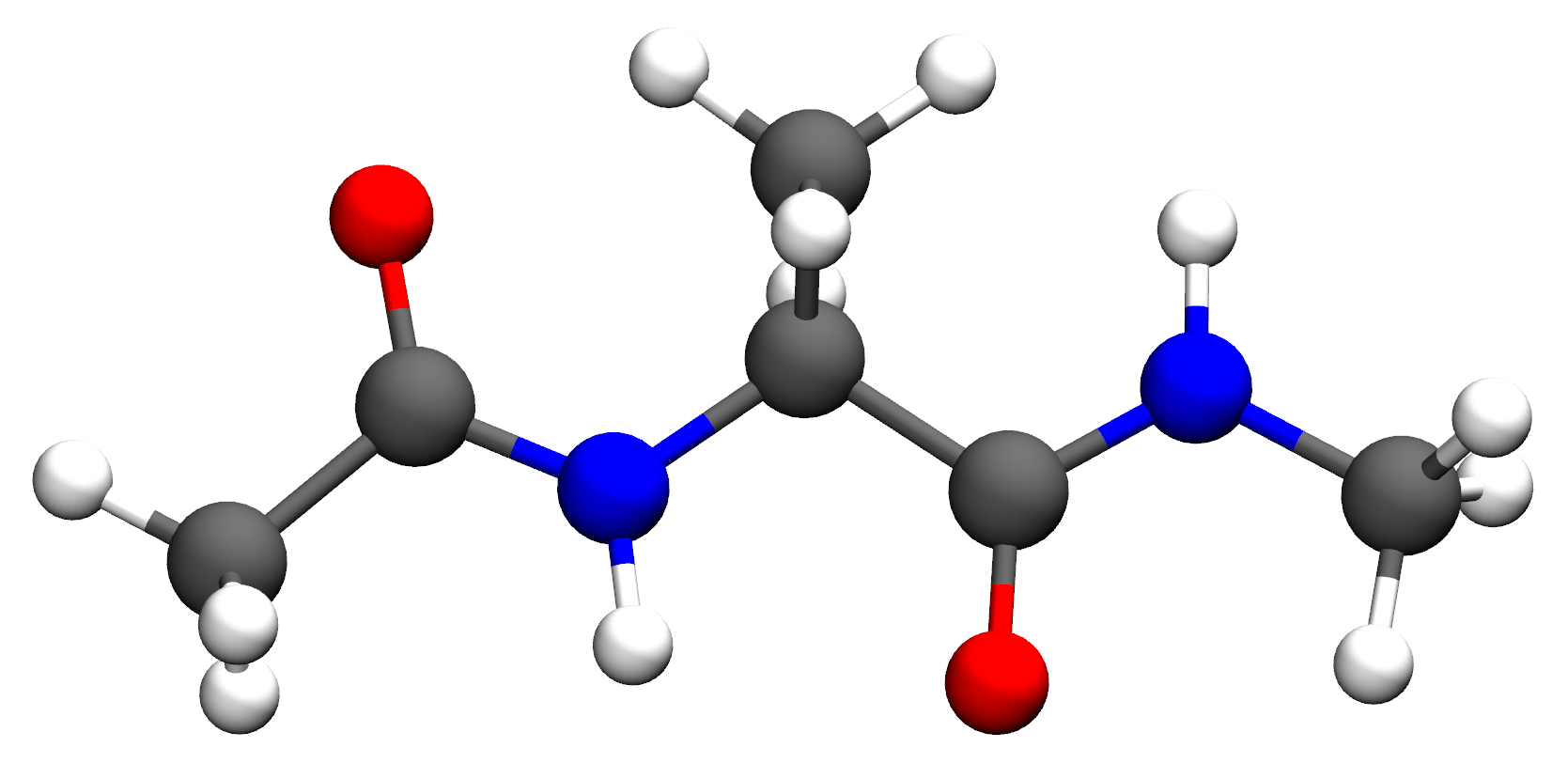}} &  & \cellcolor{blue!5}fwd KL + LDR-L2 (ours) & \cellcolor{blue!5}\textbf{\num{-214.432(0.001)e0}} & \cellcolor{blue!5}\textbf{\num{98.25(0.11)}} \\
                            &                                     & \cellcolor{blue!5}fwd KL + LDR-L1 (ours) & \cellcolor{blue!5}\num{-214.430(0.000)e0} & \cellcolor{blue!5}\num{97.81(0.07)} \\
\cmidrule(lr){2-5}
                            & \multirow{3}{*}{\shortstack[c]{\num{5e6}$\,^{\boldsymbol{\bigstar}}$}} & fwd KL & \num{-214.417(0.001)e0} & \num{94.96(0.14)} \\
                            &                                     & \cellcolor{blue!5}fwd KL + LDR-L2 (ours) & \cellcolor{blue!5}\textbf{\num{-214.437(0.000)e0}} & \cellcolor{blue!5}\textbf{\num{99.04(0.02)}} \\
                            &                                     & \cellcolor{blue!5}fwd KL + LDR-L1 (ours) & \cellcolor{blue!5}\textbf{\num{-214.437(0.001)e0}} & \cellcolor{blue!5}\textbf{\num{99.09(0.12)}} \\
\addlinespace[0.1cm]
\midrule
\multirow{3}{*}{\shortstack[c]{Alanine \\ Hexapeptide}} & \multirow{3}{*}{\shortstack[c]{\num{1e6}$\,^{\boldsymbol{\bigstar}}$}} & fwd KL & \num{-504.675(0.002)e0} & \num{13.10(0.11)} & \multirow{7}{*}{\adjincludegraphics[valign=m]{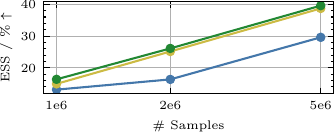}} \\
\multirow{7}{*}{\adjincludegraphics[valign=b,width=2.7cm]{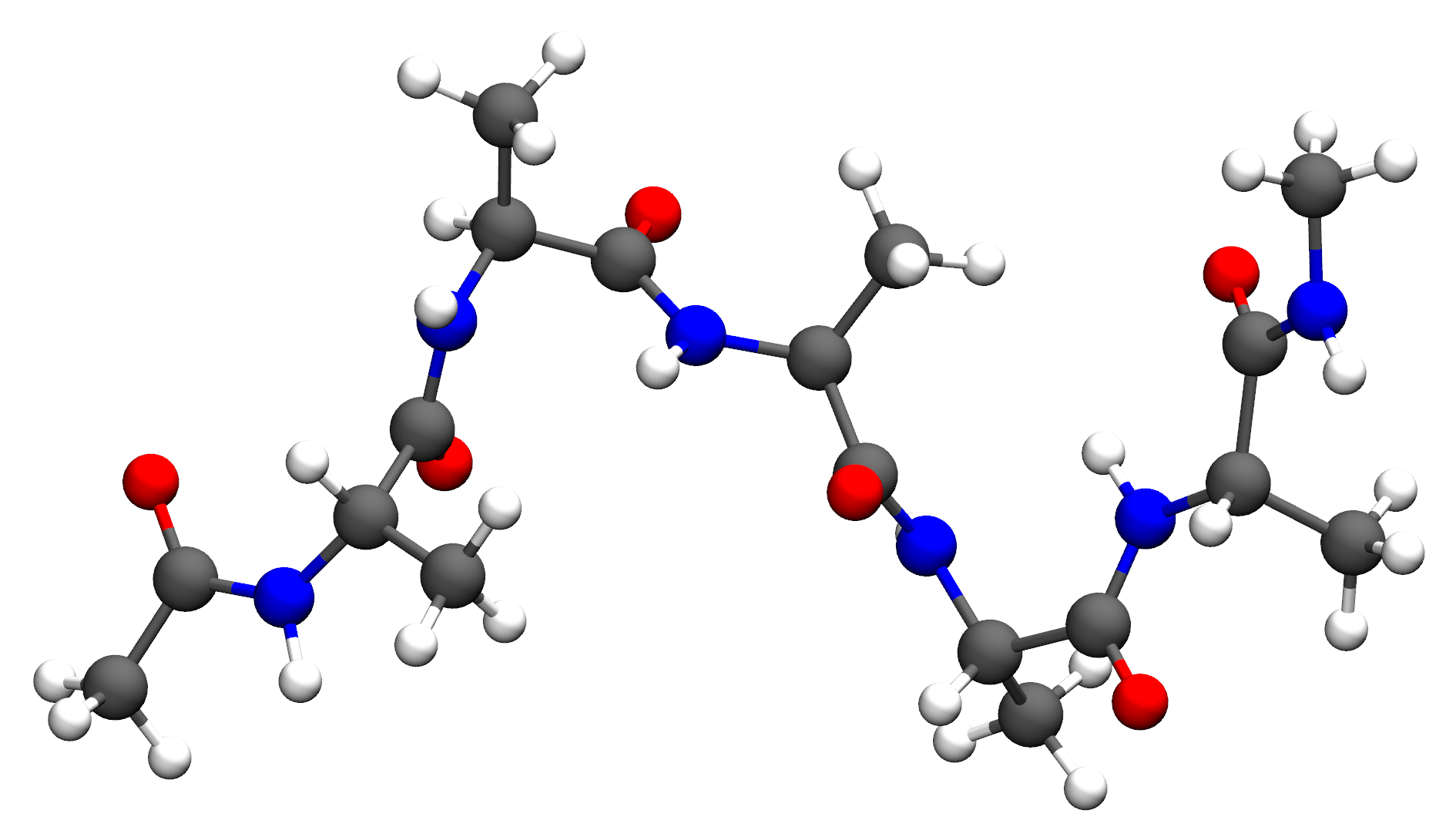}} &  & \cellcolor{blue!5}fwd KL + LDR-L2 (ours) & \cellcolor{blue!5}\num{-504.632(0.016)e0} & \cellcolor{blue!5}\num{14.92(0.23)} \\
                            &                                     & \cellcolor{blue!5}fwd KL + LDR-L1 (ours) & \cellcolor{blue!5}\textbf{\num{-504.759(0.013)e0}} & \cellcolor{blue!5}\textbf{\num{16.35(0.10)}} \\
\cmidrule(lr){2-5}
                            & \multirow{3}{*}{\shortstack[c]{\num{5e6}$\,^{\boldsymbol{\bigstar}}$}} & fwd KL & \num{-505.174(0.005)e0} & \num{29.60(0.25)} \\
                            &                                     & \cellcolor{blue!5}fwd KL + LDR-L2 (ours) & \cellcolor{blue!5}\num{-505.259(0.010)e0} & \cellcolor{blue!5}\num{38.70(0.55)} \\
                            &                                     & \cellcolor{blue!5}fwd KL + LDR-L1 (ours) & \cellcolor{blue!5}\textbf{\num{-505.308(0.007)e0}} & \cellcolor{blue!5}\textbf{\num{39.59(0.44)}} \\
\addlinespace[0.1cm]
\bottomrule
\end{tabular}}
\end{sc}
\end{small}
\end{center}
{\fontsize{7.5pt}{9pt}\selectfont $^{\boldsymbol{\bigstar}}$~We note that we performed $>\num{1e9}$ target evaluations in the MD simulations used to construct the listed datasets via downsampling (Appendix~\ref{appendix:datasets}).}
\vskip -0.1in
\end{table*}

%% file: tables/results_biased.tex
\begin{table*}
\caption{Results for training on biased datasets with importance sampling (IS).
We show results for pure data-based training with FWD KL, FWD KL + LDR on the IS samples only, and FWD KL + LDR on both the IS samples and the original biased dataset.
We report negative log-likelihood (NLL; lower is better) and effective sample size (ESS; higher is better), averaged over 4 independent experiments with standard deviations. 
\textbf{Bold} indicates metrics not significantly different from the best under 
uncorrected two-sided Welch's t-tests ($\alpha = 0.05$), capturing seed-level variability only. 
Right: ESS as a function of the number of samples drawn for importance sampling.}
\label{tab:results_biased}
\begin{center}
\begin{small}
\begin{sc}
\resizebox{\textwidth}{!}{\begin{tabular}{cclccc}
\toprule
System & {\#} IS Samples & Method & NLL $\downarrow$ & ESS / \% $\uparrow$ & \adjincludegraphics[valign=m]{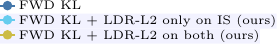} \\
\midrule
\addlinespace[0.32cm]
\multirow{2}{*}{\shortstack[c]{Alanine \\ Dipeptide}} & \multirow{6}{*}{\shortstack[c]{\num{1e5}}} & fwd KL & \num{-214.114(0.004)e0} & \num{51.42(0.42)} & \multirow{5}{*}{\adjincludegraphics[valign=m]{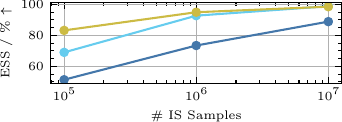}} \\
\multirow{6}{*}{\adjincludegraphics[valign=b,width=2.7cm]{./figures/molecules/aldp.png}} &  & \cellcolor{blue!5}fwd KL + LDR-L2 only on IS (ours) & \cellcolor{blue!5}\num{-214.264(0.001)e0} & \cellcolor{blue!5}\num{69.04(0.13)} \\
                                                                  &  & \cellcolor{blue!5}fwd KL + LDR-L1 only on IS (ours) & \cellcolor{blue!5}\num{-214.251(0.001)e0} & \cellcolor{blue!5}\num{66.94(0.24)} \\
                            &                                     & \cellcolor{blue!5}fwd KL + LDR-L2 on both (ours) & \cellcolor{blue!5}\textbf{\num{-214.353(0.001)e0}} & \cellcolor{blue!5}\textbf{\num{83.20(0.19)}} \\
                            &                                     & \cellcolor{blue!5}fwd KL + LDR-L1 on both (ours) & \cellcolor{blue!5}\num{-214.345(0.001)e0} & \cellcolor{blue!5}\num{82.30(0.13)} \\
\addlinespace[0.32cm]
\bottomrule
\end{tabular}}
\end{sc}
\end{small}
\end{center}
\vskip -0.1in
\end{table*}

%% file: tables/results_cmt.tex
\begin{table*}
\caption{Results for variational training without access to target samples, using only target energy evaluations.
We report negative log-likelihood (NLL; lower is better) and effective sample size (ESS; higher is better).
We average over 4 independent experiments with standard deviations for alanine dipeptide, and 3 independent experiments for alanine hexapeptide to limit computational cost.
For each number of target evaluations, \textbf{bold} indicates metrics not significantly different from the best under 
uncorrected two-sided Welch's t-tests ($\alpha = 0.05$), capturing seed-level variability only. 
Right: ESS as a function of the number of target energy evaluations.}
\label{tab:results_cmt}
\begin{center}
\begin{small}
\begin{sc}
\resizebox{\textwidth}{!}{
\begin{tabular}{clcccc}
\toprule
System & Method & {\#} Target evals $\downarrow$ & NLL $\downarrow$ & ESS / \% $\uparrow$ & \adjincludegraphics[valign=m]{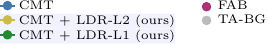} \\
\midrule
\addlinespace[0.15cm]
\multirow{2}{*}{\shortstack[c]{Alanine \\ Dipeptide}} & FAB & \num{2.13e8} & \num{-214.412(0.001)e0} & \num{94.95(0.12)} & \multirow{4}{*}{\adjincludegraphics[valign=m]{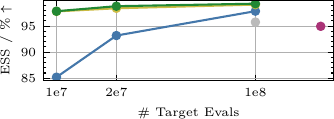}} \\
\multirow{6}{*}{\adjincludegraphics[valign=b,width=2.7cm]{./figures/molecules/aldp.png}} & TA-BG & \num{1e8} & \num{-214.419(0.002)e0} & \num{95.76(0.25)} \\
\cmidrule(lr){2-5}
                                                                  & CMT & \num{1e7} & \num{-214.358(0.003)e0} & \num{85.20(0.47)} \\
                                                                  & \cellcolor{blue!5}CMT + LDR-L2 (ours) & \cellcolor{blue!5}\num{1e7} & \cellcolor{blue!5}\textbf{\num{-214.429(0.001)e0}} & \cellcolor{blue!5}\textbf{\num{97.81(0.08)}} \\
                                                                  & \cellcolor{blue!5}CMT + LDR-L1 (ours) & \cellcolor{blue!5}\num{1e7} & \cellcolor{blue!5}\textbf{\num{-214.429(0.001)e0}} & \cellcolor{blue!5}\textbf{\num{97.84(0.05)}} \\
\midrule
\multirow{2}{*}{\shortstack[c]{Alanine \\ Hexapeptide}} & FAB & \num{4.2e8} & \num{-504.355(0.019)e0} & \num{14.41(0.27)} & \multirow{5}{*}{\adjincludegraphics[valign=m]{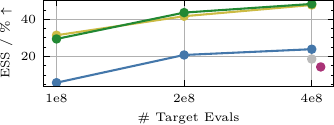}} \\
\multirow{6}{*}{\adjincludegraphics[valign=b,width=2.8cm]{./figures/molecules/hexa.png}} & TA-BG & \num{4e8} & \num{-504.782(0.019)e0} & \num{18.66(0.16)} \\
\cmidrule(lr){2-5}
                                                                    & CMT & \num{1e8} & \num{-503.190(0.051)e0} & \num{5.92(0.37)} \\
                                                                    & \cellcolor{blue!5}CMT + LDR-L2 (ours) & \cellcolor{blue!5}\num{1e8} & \cellcolor{blue!5}\textbf{\num{-504.801(0.008)e0}} & \cellcolor{blue!5}\textbf{\num{31.47(0.71)}} \\
                                                                    & \cellcolor{blue!5}CMT + LDR-L1 (ours) & \cellcolor{blue!5}\num{1e8} & \cellcolor{blue!5}\num{-504.721(0.010)e0} & \cellcolor{blue!5}\num{29.48(0.40)} \\
\addlinespace[0.15cm]
\bottomrule
\end{tabular}}
\end{sc}
\end{small}
\end{center}
\vskip -0.1in
\end{table*}

%% file: tables/results_tarflow.tex
\begin{table}[!htbp]
\caption{Results for training on Cartesian coordinates using an unbiased dataset from the target distribution of alanine dipeptide.
We average over 4 independent experiments with standard deviations. 
\textbf{Bold} indicates metrics not significantly different from the best under 
uncorrected two-sided Welch's t-tests ($\alpha = 0.05$), capturing seed-level variability only.}
\label{tab:results_tarflow}
\begin{center}
\begin{small}
\begin{sc}
\resizebox{\columnwidth}{!}{\begin{tabular}{lcccccc}
\toprule
Method & NLL $\downarrow$ & ESS / \% $\uparrow$\\
\midrule FWD KL {\normalfont\cite{tanScalableEquilibriumSampling2025a}} & \num{-219.583(0.051)e0} & \num{19.40(0.80)} \\
FWD KL$^{\boldsymbol{\bigstar}}$ & \num{-219.583(0.050)e0} & \num{27.32(1.28)} \\
\cellcolor{blue!5}FWD KL$^{\boldsymbol{\bigstar}}$ + LDR-L2 (ours) & \cellcolor{blue!5}\textbf{\num{-219.815(0.034)e0}} & \cellcolor{blue!5}\textbf{\num{35.89(1.69)}} \\
\cellcolor{blue!5}FWD KL$^{\boldsymbol{\bigstar}}$ + LDR-L1 (ours) & \cellcolor{blue!5}\textbf{\num{-219.800(0.044)e0}} & \cellcolor{blue!5}\textbf{\num{35.75(1.63)}} \\
\bottomrule
\end{tabular}}
\end{sc}
\end{small}
\end{center}
{\fontsize{7.5pt}{9pt}\selectfont $^{\boldsymbol{\bigstar}}$~Uses our improved augmentation correction, see Appendix~\ref{appendix:theory:com_augm}.}
\vskip -0.1in
\end{table}

%% file: tables/inference_speeds.tex
\begin{table}[!htbp]
\caption{Approximate sampling speed for different architecture / system
combinations, evaluated including the calculation of importance weights. For each
architecture / system combination, we increased the batch size until either
running out of memory or until the throughput did not increase further.
Benchmarks were performed on a single NVIDIA A100 GPU (80\,GB memory).}
\label{SI:tab_inference_speed}
\begin{center}
\begin{small}
\begin{sc}
\begin{tabular}{cccc}
\toprule
System & Architecture & Coordinate representation & Samples / hour \\
\midrule
Alanine Dipeptide & Spline Flow & \normalfont ICs & $\sim$ \num{2.8e8} \\
Alanine Hexapeptide & Spline Flow & \normalfont ICs & $\sim$ \num{1.1e8} \\
Alanine Dipeptide & TarFlow & \normalfont Cartesian & $\sim$ \num{2.8e7} \\
\midrule
Alanine Dipeptide & \normalfont \cite{garciasatorrasEquivariantNormalizingFlows2021} & \normalfont Cartesian & $\sim 7300$ \\
\bottomrule
\end{tabular}
\end{sc}
\end{small}
\end{center}
\vskip -0.1in
\end{table}

%% file: tables/params_architecture_systems.tex
\begin{table}[!htbp]
\caption{Summary of the number of trainable parameters for each architecture / system combination.}
\label{SI:tab_num_parameters}
\begin{center}
\begin{small}
\begin{sc}
\begin{tabular}{cccc}
\toprule
System & Architecture & Coordinate representation & No. parameters \\
\midrule
GMM & RealNVP & - & \num{403230} \\
Alanine Dipeptide & Spline Flow & \normalfont ICs & \num{7421512} \\
Alanine Hexapeptide & Spline Flow & \normalfont ICs & \num{12124616} \\
Alanine Dipeptide & TarFlow & \normalfont Cartesian & \num{12668952} \\
\bottomrule
\end{tabular}
\end{sc}
\end{small}
\end{center}
\vskip -0.1in
\end{table}

%% file: tables/training_times.tex
\begin{table}[!htbp]
\caption{Summary of the number of gradient descent steps and the total training time (excluding evaluation) for each method.}
\label{SI:tab_training_times}
\begin{center}
\begin{small}
\begin{sc}
\begin{tabular}{ccccc}
\toprule
System & Architecture (Rep.) & Method & No. gradient steps & Training time \\
\midrule
GMM & RealNVP & unbiased dataset & \num{10000} & $\sim$ \SI{4}{\minute} \\
Alanine dipeptide & Spline Flow (IC) & unbiased dataset & \num{250000} & $\sim$ \SI{11.8}{\hour} \\
Alanine hexapeptide & Spline Flow (IC) & unbiased dataset & \num{250000} & $\sim$ \SI{17.3}{\hour} \\
\midrule
Alanine dipeptide & Spline Flow (IC) & biased dataset & \num{250000} & $\sim$ \SI{12.6}{\hour} \\
\midrule
Alanine dipeptide & Spline Flow (IC) & CMT (\num{1e7} evals) & \num{400000} & $\sim$ \SI{17.9}{\hour} \\
Alanine hexapeptide & Spline Flow (IC) & CMT (\num{1e8} evals) & \num{800000} & $\sim$ \SI{53.0}{\hour} \\
\midrule
Alanine dipeptide & TarFlow (Cart.) & unbiased dataset & \num{400000} & $\sim$ \SI{18.6}{\hour} \\
\bottomrule
\end{tabular}
\end{sc}
\end{small}
\end{center}
\vskip -0.1in
\end{table}

%% file: tables/results_additional_ld_variants.tex
\begin{table*}[!htbp]
\caption{Comparison of additional log-dispersion variants for unbiased training on alanine hexapeptide using \num{1e6} samples.
We average over 4 independent experiments with standard deviations. 
\textbf{Bold} indicates metrics not significantly different from the best under 
uncorrected two-sided Welch's t-tests ($\alpha = 0.05$), capturing seed-level variability only.}
\label{tab:results_additional_ld_variants}
\begin{center}
\begin{small}
\begin{sc}
\begin{tabular}{lcc}
\toprule
Method & NLL $\downarrow$ & ESS / \% $\uparrow$\\
\midrule FWD KL & \num{-504.675(0.002)e0} & \num{13.10(0.11)} \\
FWD KL + LDR-L1 (ours) & \textbf{\num{-504.759(0.013)e0}} & \textbf{\num{16.35(0.10)}} \\
FWD KL + LDR-L1.5 (ours) & \textbf{\num{-504.750(0.007)e0}} & \textbf{\num{16.74(0.20)}} \\
FWD KL + LDR-L2 (ours) & \num{-504.632(0.016)e0} & \num{14.92(0.23)} \\
FWD KL + LDR-L3 (ours) & \num{-504.425(0.020)e0} & \num{8.95(0.26)} \\
FWD KL + LDR-Huber (ours) & \textbf{\num{-504.752(0.016)e0}} & \textbf{\num{16.88(0.40)}} \\
\bottomrule
\end{tabular}
\end{sc}
\end{small}
\end{center}
\vskip -0.1in
\end{table*}

%% file: tables/results_unbiased_all_GMM.tex
\begin{table*}[!htbp]
\caption{Results for training on unbiased datasets from the target distribution of the GMM system.
We average over 4 independent experiments with standard deviations. 
For each number of samples, \textbf{bold} indicates metrics not significantly different from the best under 
uncorrected two-sided Welch's t-tests ($\alpha = 0.05$), capturing seed-level variability only.}
\label{tab:results_unbiased_all_GMM}
\begin{center}
\begin{small}
\begin{sc}
\resizebox{0.5\textwidth}{!}{\begin{tabular}{cclcccccc}
\toprule
System & {\#} Samples & Method & NLL $\downarrow$ & ESS / \% $\uparrow$\\
\midrule \multirow{14}{*}{\shortstack[c]{GMM}} & \multirow{4}{*}{500} & FWD KL & \num{2.850(0.061)e0} & \num{77.76(6.90)} \\
& & FWD KL (PATH GRADIENTS) & \num{2.812(0.008)e0} & \num{77.97(5.32)} \\
& & \cellcolor{blue!5}FWD KL + LDR-L2 (ours) & \cellcolor{blue!5}\textbf{\num{2.736(0.003)e0}} & \cellcolor{blue!5}\textbf{\num{95.26(2.48)}} \\
& & \cellcolor{blue!5}FWD KL + LDR-L1 (ours) & \cellcolor{blue!5}\textbf{\num{2.734(0.004)e0}} & \cellcolor{blue!5}\textbf{\num{96.98(0.60)}} \\
\cmidrule(lr){2-5} & \multirow{4}{*}{1000} & FWD KL & \num{2.791(0.013)e0} & \num{86.47(1.43)} \\
& & FWD KL (PATH GRADIENTS) & \num{2.777(0.027)e0} & \num{86.82(5.43)} \\
& & \cellcolor{blue!5}FWD KL + LDR-L2 (ours) & \cellcolor{blue!5}\textbf{\num{2.725(0.001)e0}} & \cellcolor{blue!5}\textbf{\num{98.62(0.31)}} \\
& & \cellcolor{blue!5}FWD KL + LDR-L1 (ours) & \cellcolor{blue!5}\textbf{\num{2.727(0.004)e0}} & \cellcolor{blue!5}\textbf{\num{96.83(2.60)}} \\
\cmidrule(lr){2-5} & \multirow{4}{*}{10000} & FWD KL & \num{2.733(0.003)e0} & \num{96.75(0.46)} \\
& & FWD KL (PATH GRADIENTS) & \num{2.726(0.002)e0} & \textbf{\num{97.54(1.44)}} \\
& & \cellcolor{blue!5}FWD KL + LDR-L2 (ours) & \cellcolor{blue!5}\textbf{\num{2.730(0.014)e0}} & \cellcolor{blue!5}\textbf{\num{97.54(2.50)}} \\
& & \cellcolor{blue!5}FWD KL + LDR-L1 (ours) & \cellcolor{blue!5}\textbf{\num{2.720(0.001)e0}} & \cellcolor{blue!5}\textbf{\num{99.47(0.29)}} \\
\bottomrule
\end{tabular}}
\end{sc}
\end{small}
\end{center}
\vskip -0.1in
\end{table*}

%% file: tables/results_unbiased_all_molecular.tex
\begin{table*}[!htbp]
\caption{Results for training on unbiased datasets from the target distribution of alanine dipeptide and alanine hexapeptide.
We average over 4 independent experiments with standard deviations. 
For each number of samples, \textbf{bold} indicates metrics not significantly different from the best under 
uncorrected two-sided Welch's t-tests ($\alpha = 0.05$), capturing seed-level variability only.}
\label{tab:results_unbiased_all_molecular}
\begin{center}
\begin{small}
\begin{sc}
\resizebox{\textwidth}{!}{\begin{tabular}{cclccccccc}
\toprule
System & {\#} Samples & Method & NLL $\downarrow$ & ESS / \% $\uparrow$ & RAM KL $\downarrow$ & RAM KL W. RW $\downarrow$ & TICA KL $\downarrow$ & TICA KL W. RW $\downarrow$ & $\mathcal{E}$-$\mathcal{W}_2$\\
\midrule \multirow{11}{*}{\shortstack[c]{Alanine \\ Dipeptide}} & \multirow{3}{*}{\num{1e6}} & FWD KL & \num{-214.378(0.000)e0} & \num{87.94(0.06)} & \num{2.57(0.12)e-3} & \num{1.96(0.04)e-3} & \textbf{\num{1.18(0.18)e-3}} & \num{1.05(0.07)e-3} & \textbf{\num{5.99(1.47)e-3}} \\
& & \cellcolor{blue!5}FWD KL + LDR-L2 (ours) & \cellcolor{blue!5}\textbf{\num{-214.432(0.001)e0}} & \cellcolor{blue!5}\textbf{\num{98.25(0.11)}} & \cellcolor{blue!5}\textbf{\num{1.60(0.05)e-3}} & \cellcolor{blue!5}\textbf{\num{1.34(0.03)e-3}} & \cellcolor{blue!5}\textbf{\num{9.58(0.20)e-4}} & \cellcolor{blue!5}\textbf{\num{7.05(0.34)e-4}} & \cellcolor{blue!5}\textbf{\num{4.96(0.64)e-3}} \\
& & \cellcolor{blue!5}FWD KL + LDR-L1 (ours) & \cellcolor{blue!5}\num{-214.430(0.000)e0} & \cellcolor{blue!5}\num{97.81(0.07)} & \cellcolor{blue!5}\textbf{\num{1.64(0.04)e-3}} & \cellcolor{blue!5}\num{1.55(0.03)e-3} & \cellcolor{blue!5}\textbf{\num{9.71(0.91)e-4}} & \cellcolor{blue!5}\num{8.40(0.87)e-4} & \cellcolor{blue!5}\textbf{\num{5.88(0.51)e-3}} \\
\cmidrule(lr){2-10} & \multirow{3}{*}{\num{2e6}} & FWD KL & \num{-214.396(0.001)e0} & \num{91.06(0.22)} & \num{1.79(0.05)e-3} & \num{1.68(0.06)e-3} & \textbf{\num{8.74(0.48)e-4}} & \num{9.21(0.42)e-4} & \textbf{\num{6.32(1.93)e-3}} \\
& & \cellcolor{blue!5}FWD KL + LDR-L2 (ours) & \cellcolor{blue!5}\textbf{\num{-214.435(0.001)e0}} & \cellcolor{blue!5}\textbf{\num{98.80(0.11)}} & \cellcolor{blue!5}\textbf{\num{1.42(0.04)e-3}} & \cellcolor{blue!5}\textbf{\num{1.18(0.04)e-3}} & \cellcolor{blue!5}\textbf{\num{9.10(0.94)e-4}} & \cellcolor{blue!5}\textbf{\num{6.51(0.31)e-4}} & \cellcolor{blue!5}\textbf{\num{6.32(1.17)e-3}} \\
& & \cellcolor{blue!5}FWD KL + LDR-L1 (ours) & \cellcolor{blue!5}\textbf{\num{-214.435(0.000)e0}} & \cellcolor{blue!5}\textbf{\num{98.81(0.04)}} & \cellcolor{blue!5}\textbf{\num{1.48(0.02)e-3}} & \cellcolor{blue!5}\num{1.36(0.04)e-3} & \cellcolor{blue!5}\textbf{\num{8.62(0.49)e-4}} & \cellcolor{blue!5}\num{7.18(0.34)e-4} & \cellcolor{blue!5}\textbf{\num{6.00(1.86)e-3}} \\
\cmidrule(lr){2-10} & \multirow{3}{*}{\num{5e6}} & FWD KL & \num{-214.417(0.001)e0} & \num{94.96(0.14)} & \num{1.49(0.02)e-3} & \num{1.42(0.03)e-3} & \textbf{\num{7.21(0.20)e-4}} & \textbf{\num{7.17(0.32)e-4}} & \textbf{\num{5.83(1.54)e-3}} \\
& & \cellcolor{blue!5}FWD KL + LDR-L2 (ours) & \cellcolor{blue!5}\textbf{\num{-214.437(0.000)e0}} & \cellcolor{blue!5}\textbf{\num{99.04(0.02)}} & \cellcolor{blue!5}\textbf{\num{1.34(0.03)e-3}} & \cellcolor{blue!5}\textbf{\num{1.28(0.03)e-3}} & \cellcolor{blue!5}\num{7.87(0.42)e-4} & \cellcolor{blue!5}\textbf{\num{6.92(0.26)e-4}} & \cellcolor{blue!5}\textbf{\num{5.23(0.44)e-3}} \\
& & \cellcolor{blue!5}FWD KL + LDR-L1 (ours) & \cellcolor{blue!5}\textbf{\num{-214.437(0.001)e0}} & \cellcolor{blue!5}\textbf{\num{99.09(0.12)}} & \cellcolor{blue!5}\textbf{\num{1.38(0.04)e-3}} & \cellcolor{blue!5}\textbf{\num{1.29(0.04)e-3}} & \cellcolor{blue!5}\num{7.94(0.13)e-4} & \cellcolor{blue!5}\textbf{\num{6.91(0.26)e-4}} & \cellcolor{blue!5}\textbf{\num{6.37(1.06)e-3}} \\
\midrule \multirow{11}{*}{\shortstack[c]{Alanine \\ Hexapeptide}} & \multirow{3}{*}{\num{1e6}} & FWD KL & \num{-504.675(0.002)e0} & \num{13.10(0.11)} & \textbf{\num{3.38(0.52)e-3}} & \num{6.50(0.18)e-3} & \textbf{\num{4.47(0.60)e-3}} & \num{1.04(0.06)e-2} & \num{6.93(0.93)e-2} \\
& & \cellcolor{blue!5}FWD KL + LDR-L2 (ours) & \cellcolor{blue!5}\num{-504.632(0.016)e0} & \cellcolor{blue!5}\num{14.92(0.23)} & \cellcolor{blue!5}\num{1.29(0.09)e-2} & \cellcolor{blue!5}\textbf{\num{4.31(0.04)e-3}} & \cellcolor{blue!5}\num{2.37(0.22)e-2} & \cellcolor{blue!5}\textbf{\num{7.01(0.30)e-3}} & \cellcolor{blue!5}\textbf{\num{4.75(1.19)e-2}} \\
& & \cellcolor{blue!5}FWD KL + LDR-L1 (ours) & \cellcolor{blue!5}\textbf{\num{-504.759(0.013)e0}} & \cellcolor{blue!5}\textbf{\num{16.35(0.10)}} & \cellcolor{blue!5}\num{5.04(0.30)e-3} & \cellcolor{blue!5}\num{4.89(0.13)e-3} & \cellcolor{blue!5}\num{8.49(0.40)e-3} & \cellcolor{blue!5}\num{7.86(0.23)e-3} & \cellcolor{blue!5}\textbf{\num{5.52(1.54)e-2}} \\
\cmidrule(lr){2-10} & \multirow{3}{*}{\num{2e6}} & FWD KL & \num{-504.813(0.008)e0} & \num{16.33(0.19)} & \num{3.06(0.10)e-3} & \num{5.80(0.17)e-3} & \textbf{\num{4.29(0.23)e-3}} & \num{9.36(0.26)e-3} & \textbf{\num{5.67(0.56)e-2}} \\
& & \cellcolor{blue!5}FWD KL + LDR-L2 (ours) & \cellcolor{blue!5}\num{-505.034(0.007)e0} & \cellcolor{blue!5}\num{25.13(0.24)} & \cellcolor{blue!5}\num{3.66(0.17)e-3} & \cellcolor{blue!5}\textbf{\num{3.26(0.03)e-3}} & \cellcolor{blue!5}\num{5.95(0.38)e-3} & \cellcolor{blue!5}\textbf{\num{5.24(0.09)e-3}} & \cellcolor{blue!5}\textbf{\num{5.17(0.69)e-2}} \\
& & \cellcolor{blue!5}FWD KL + LDR-L1 (ours) & \cellcolor{blue!5}\textbf{\num{-505.066(0.001)e0}} & \cellcolor{blue!5}\textbf{\num{26.10(0.12)}} & \cellcolor{blue!5}\textbf{\num{2.84(0.12)e-3}} & \cellcolor{blue!5}\num{3.35(0.03)e-3} & \cellcolor{blue!5}\textbf{\num{4.52(0.13)e-3}} & \cellcolor{blue!5}\num{5.50(0.06)e-3} & \cellcolor{blue!5}\textbf{\num{4.71(0.87)e-2}} \\
\cmidrule(lr){2-10} & \multirow{3}{*}{\num{5e6}} & FWD KL & \num{-505.174(0.005)e0} & \num{29.60(0.25)} & \num{1.89(0.04)e-3} & \num{3.39(0.01)e-3} & \textbf{\num{2.73(0.12)e-3}} & \num{5.39(0.07)e-3} & \textbf{\num{3.66(0.63)e-2}} \\
& & \cellcolor{blue!5}FWD KL + LDR-L2 (ours) & \cellcolor{blue!5}\num{-505.259(0.010)e0} & \cellcolor{blue!5}\num{38.70(0.55)} & \cellcolor{blue!5}\num{3.30(0.38)e-3} & \cellcolor{blue!5}\textbf{\num{2.44(0.05)e-3}} & \cellcolor{blue!5}\num{5.64(0.51)e-3} & \cellcolor{blue!5}\textbf{\num{4.26(0.10)e-3}} & \cellcolor{blue!5}\textbf{\num{4.18(0.28)e-2}} \\
& & \cellcolor{blue!5}FWD KL + LDR-L1 (ours) & \cellcolor{blue!5}\textbf{\num{-505.308(0.007)e0}} & \cellcolor{blue!5}\textbf{\num{39.59(0.44)}} & \cellcolor{blue!5}\textbf{\num{1.73(0.05)e-3}} & \cellcolor{blue!5}\num{2.63(0.05)e-3} & \cellcolor{blue!5}\textbf{\num{2.90(0.09)e-3}} & \cellcolor{blue!5}\num{4.49(0.10)e-3} & \cellcolor{blue!5}\textbf{\num{4.38(0.73)e-2}} \\
\bottomrule
\end{tabular}}
\end{sc}
\end{small}
\end{center}
\vskip -0.1in
\end{table*}

%% file: tables/results_biased_all.tex
\begin{table*}[!htbp]
\caption{Results for training on biased datasets using importance sampling.
We average over 4 independent experiments with standard deviations. 
For each number of IS samples, \textbf{bold} indicates metrics not significantly different from the best under 
uncorrected two-sided Welch's t-tests ($\alpha = 0.05$), capturing seed-level variability only.}
\label{tab:results_biased_all}
\begin{center}
\begin{small}
\begin{sc}
\resizebox{\textwidth}{!}{\begin{tabular}{cclccccccc}
\toprule
System & {\#} IS Samples & Method & NLL $\downarrow$ & ESS / \% $\uparrow$ & RAM KL $\downarrow$ & RAM KL W. RW $\downarrow$ & TICA KL $\downarrow$ & TICA KL W. RW $\downarrow$ & $\mathcal{E}$-$\mathcal{W}_2$\\
\midrule \multirow{17}{*}{\shortstack[c]{Alanine \\ Dipeptide}} & \multirow{5}{*}{\num{1e5}} & FWD KL & \num{-214.114(0.004)e0} & \num{51.42(0.42)} & \num{1.61(0.07)e-2} & \num{2.21(0.29)e-3} & \num{4.54(0.34)e-3} & \num{1.19(0.11)e-3} & \num{1.25(0.40)e-2} \\
& & \cellcolor{blue!5}FWD KL + LDR-L2 only on IS (ours) & \cellcolor{blue!5}\num{-214.264(0.001)e0} & \cellcolor{blue!5}\num{69.04(0.13)} & \cellcolor{blue!5}\num{7.17(0.69)e-3} & \cellcolor{blue!5}\num{1.49(0.09)e-3} & \cellcolor{blue!5}\num{2.79(0.36)e-3} & \cellcolor{blue!5}\num{7.71(0.32)e-4} & \cellcolor{blue!5}\num{7.53(0.80)e-3} \\
& & \cellcolor{blue!5}FWD KL + LDR-L1 only on IS (ours) & \cellcolor{blue!5}\num{-214.251(0.001)e0} & \cellcolor{blue!5}\num{66.94(0.24)} & \cellcolor{blue!5}\num{6.69(0.15)e-3} & \cellcolor{blue!5}\num{1.62(0.11)e-3} & \cellcolor{blue!5}\num{2.93(0.28)e-3} & \cellcolor{blue!5}\num{8.22(0.60)e-4} & \cellcolor{blue!5}\textbf{\num{1.01(0.30)e-2}} \\
& & \cellcolor{blue!5}FWD KL + LDR-L2 on both (ours) & \cellcolor{blue!5}\textbf{\num{-214.353(0.001)e0}} & \cellcolor{blue!5}\textbf{\num{83.20(0.19)}} & \cellcolor{blue!5}\textbf{\num{2.57(0.11)e-3}} & \cellcolor{blue!5}\num{1.29(0.03)e-3} & \cellcolor{blue!5}\textbf{\num{1.22(0.06)e-3}} & \cellcolor{blue!5}\num{6.77(0.44)e-4} & \cellcolor{blue!5}\textbf{\num{5.76(0.80)e-3}} \\
& & \cellcolor{blue!5}FWD KL + LDR-L1 on both (ours) & \cellcolor{blue!5}\num{-214.345(0.001)e0} & \cellcolor{blue!5}\num{82.30(0.13)} & \cellcolor{blue!5}\num{3.26(0.19)e-3} & \cellcolor{blue!5}\textbf{\num{1.23(0.03)e-3}} & \cellcolor{blue!5}\num{1.51(0.19)e-3} & \cellcolor{blue!5}\textbf{\num{6.10(0.13)e-4}} & \cellcolor{blue!5}\textbf{\num{7.00(1.94)e-3}} \\
\cmidrule(lr){2-10} & \multirow{5}{*}{\num{1e6}} & FWD KL & \num{-214.287(0.001)e0} & \num{73.45(0.15)} & \num{5.75(0.32)e-3} & \num{2.56(0.24)e-3} & \num{2.30(0.26)e-3} & \num{1.37(0.20)e-3} & \textbf{\num{6.50(1.46)e-3}} \\
& & \cellcolor{blue!5}FWD KL + LDR-L2 only on IS (ours) & \cellcolor{blue!5}\num{-214.406(0.000)e0} & \cellcolor{blue!5}\num{92.69(0.03)} & \cellcolor{blue!5}\num{1.88(0.02)e-3} & \cellcolor{blue!5}\textbf{\num{1.36(0.07)e-3}} & \cellcolor{blue!5}\num{1.10(0.03)e-3} & \cellcolor{blue!5}\num{7.37(0.31)e-4} & \cellcolor{blue!5}\textbf{\num{6.26(1.07)e-3}} \\
& & \cellcolor{blue!5}FWD KL + LDR-L1 only on IS (ours) & \cellcolor{blue!5}\num{-214.406(0.000)e0} & \cellcolor{blue!5}\num{93.38(0.03)} & \cellcolor{blue!5}\num{1.73(0.02)e-3} & \cellcolor{blue!5}\textbf{\num{1.27(0.04)e-3}} & \cellcolor{blue!5}\num{1.04(0.03)e-3} & \cellcolor{blue!5}\textbf{\num{6.73(0.28)e-4}} & \cellcolor{blue!5}\textbf{\num{5.38(0.66)e-3}} \\
& & \cellcolor{blue!5}FWD KL + LDR-L2 on both (ours) & \cellcolor{blue!5}\textbf{\num{-214.416(0.000)e0}} & \cellcolor{blue!5}\num{94.91(0.05)} & \cellcolor{blue!5}\textbf{\num{1.61(0.04)e-3}} & \cellcolor{blue!5}\textbf{\num{1.31(0.05)e-3}} & \cellcolor{blue!5}\textbf{\num{8.48(0.71)e-4}} & \cellcolor{blue!5}\textbf{\num{6.86(0.65)e-4}} & \cellcolor{blue!5}\textbf{\num{5.49(1.24)e-3}} \\
& & \cellcolor{blue!5}FWD KL + LDR-L1 on both (ours) & \cellcolor{blue!5}\num{-214.416(0.000)e0} & \cellcolor{blue!5}\textbf{\num{95.28(0.05)}} & \cellcolor{blue!5}\textbf{\num{1.60(0.01)e-3}} & \cellcolor{blue!5}\textbf{\num{1.28(0.01)e-3}} & \cellcolor{blue!5}\textbf{\num{8.61(0.34)e-4}} & \cellcolor{blue!5}\textbf{\num{6.42(0.22)e-4}} & \cellcolor{blue!5}\textbf{\num{6.78(1.84)e-3}} \\
\cmidrule(lr){2-10} & \multirow{5}{*}{\num{1e7}} & FWD KL & \num{-214.383(0.000)e0} & \num{88.85(0.03)} & \num{1.91(0.05)e-3} & \num{1.64(0.04)e-3} & \num{8.97(0.36)e-4} & \num{8.91(0.27)e-4} & \textbf{\num{6.08(1.49)e-3}} \\
& & \cellcolor{blue!5}FWD KL + LDR-L2 only on IS (ours) & \cellcolor{blue!5}\textbf{\num{-214.436(0.000)e0}} & \cellcolor{blue!5}\textbf{\num{98.78(0.01)}} & \cellcolor{blue!5}\textbf{\num{1.37(0.05)e-3}} & \cellcolor{blue!5}\textbf{\num{1.30(0.04)e-3}} & \cellcolor{blue!5}\textbf{\num{8.24(0.49)e-4}} & \cellcolor{blue!5}\textbf{\num{7.17(0.48)e-4}} & \cellcolor{blue!5}\textbf{\num{5.29(1.19)e-3}} \\
& & \cellcolor{blue!5}FWD KL + LDR-L1 only on IS (ours) & \cellcolor{blue!5}\num{-214.434(0.000)e0} & \cellcolor{blue!5}\num{98.44(0.01)} & \cellcolor{blue!5}\textbf{\num{1.37(0.05)e-3}} & \cellcolor{blue!5}\textbf{\num{1.25(0.05)e-3}} & \cellcolor{blue!5}\textbf{\num{8.07(0.32)e-4}} & \cellcolor{blue!5}\textbf{\num{6.73(0.39)e-4}} & \cellcolor{blue!5}\textbf{\num{5.63(1.53)e-3}} \\
& & \cellcolor{blue!5}FWD KL + LDR-L2 on both (ours) & \cellcolor{blue!5}\num{-214.434(0.000)e0} & \cellcolor{blue!5}\num{98.48(0.02)} & \cellcolor{blue!5}\textbf{\num{1.41(0.05)e-3}} & \cellcolor{blue!5}\num{1.34(0.05)e-3} & \cellcolor{blue!5}\textbf{\num{8.26(0.52)e-4}} & \cellcolor{blue!5}\textbf{\num{6.99(0.20)e-4}} & \cellcolor{blue!5}\textbf{\num{6.31(2.38)e-3}} \\
& & \cellcolor{blue!5}FWD KL + LDR-L1 on both (ours) & \cellcolor{blue!5}\num{-214.431(0.000)e0} & \cellcolor{blue!5}\num{97.97(0.03)} & \cellcolor{blue!5}\num{1.45(0.04)e-3} & \cellcolor{blue!5}\textbf{\num{1.31(0.04)e-3}} & \cellcolor{blue!5}\textbf{\num{8.11(0.22)e-4}} & \cellcolor{blue!5}\textbf{\num{6.73(0.21)e-4}} & \cellcolor{blue!5}\textbf{\num{6.64(1.43)e-3}} \\
\bottomrule
\end{tabular}}
\end{sc}
\end{small}
\end{center}
\vskip -0.1in
\end{table*}

%% file: tables/results_CMT_all.tex
\begin{table*}[!htbp]
\caption{Results for variational training without data from the target distribution.
We average over 4 independent experiments with standard deviations for alanine dipeptide and 3 independent experiments for alanine hexapeptide.
For each number of target evaluations, \textbf{bold} indicates metrics not significantly different from the best under 
uncorrected two-sided Welch's t-tests ($\alpha = 0.05$), capturing seed-level variability only.}
\label{tab:results_CMT_all}
\begin{center}
\begin{small}
\begin{sc}
\resizebox{\textwidth}{!}{\begin{tabular}{cclcccccccc}
\toprule
System & {\#} Target Evals & Method & NLL $\downarrow$ & ESS / \% $\uparrow$ & RAM KL $\downarrow$ & RAM KL W. RW $\downarrow$ & TICA KL $\downarrow$ & TICA KL W. RW $\downarrow$ & $\mathcal{E}$-$\mathcal{W}_2$\\
\midrule \multirow{3}{*}{} & \multirow{1}{*}{\num{2.13e8}} & FAB & \num{-214.412(0.001)e0} & \num{94.95(0.12)} & \num{1.51(0.04)e-3} & \num{1.22(0.03)e-3} & \num{3.05(0.40)e-3} & \num{6.30(0.12)e-4} & \num{5.89(0.77)e-3} \\
\multirow{3}{*}{} & \multirow{1}{*}{\num{1e8}} & TA-BG & \num{-214.419(0.002)e0} & \num{95.76(0.25)} & \num{1.91(0.13)e-3} & \num{1.34(0.06)e-3} & \num{2.72(0.80)e-3} & \num{1.08(0.17)e-3} & \num{8.70(1.71)e-3} \\
\cmidrule(lr){2-10}\multirow{11}{*}{\shortstack[c]{Alanine \\ Dipeptide}} & \multirow{3}{*}{\num{1e7}} & CMT & \num{-214.358(0.003)e0} & \num{85.20(0.47)} & \num{3.78(0.38)e-3} & \num{2.33(0.23)e-3} & \num{2.16(0.71)e-3} & \num{1.17(0.12)e-3} & \textbf{\num{5.36(1.20)e-3}} \\
& & \cellcolor{blue!5}CMT + LDR-L2 (ours) & \cellcolor{blue!5}\textbf{\num{-214.429(0.001)e0}} & \cellcolor{blue!5}\textbf{\num{97.81(0.08)}} & \cellcolor{blue!5}\textbf{\num{1.62(0.05)e-3}} & \cellcolor{blue!5}\textbf{\num{1.41(0.04)e-3}} & \cellcolor{blue!5}\textbf{\num{9.36(0.07)e-4}} & \cellcolor{blue!5}\textbf{\num{7.43(0.40)e-4}} & \cellcolor{blue!5}\textbf{\num{5.66(0.26)e-3}} \\
& & \cellcolor{blue!5}CMT + LDR-L1 (ours) & \cellcolor{blue!5}\textbf{\num{-214.429(0.001)e0}} & \cellcolor{blue!5}\textbf{\num{97.84(0.05)}} & \cellcolor{blue!5}\textbf{\num{1.57(0.09)e-3}} & \cellcolor{blue!5}\textbf{\num{1.42(0.08)e-3}} & \cellcolor{blue!5}\textbf{\num{9.47(0.44)e-4}} & \cellcolor{blue!5}\textbf{\num{7.46(0.70)e-4}} & \cellcolor{blue!5}\textbf{\num{5.70(0.65)e-3}} \\
\cmidrule(lr){2-10} & \multirow{3}{*}{\num{2e7}} & CMT & \num{-214.405(0.001)e0} & \num{93.18(0.19)} & \num{1.70(0.03)e-3} & \num{1.52(0.03)e-3} & \textbf{\num{1.40(0.33)e-3}} & \textbf{\num{7.94(0.68)e-4}} & \textbf{\num{4.84(0.90)e-3}} \\
& & \cellcolor{blue!5}CMT + LDR-L2 (ours) & \cellcolor{blue!5}\num{-214.432(0.000)e0} & \cellcolor{blue!5}\num{98.41(0.03)} & \cellcolor{blue!5}\textbf{\num{1.38(0.03)e-3}} & \cellcolor{blue!5}\textbf{\num{1.29(0.06)e-3}} & \cellcolor{blue!5}\num{1.95(0.10)e-3} & \cellcolor{blue!5}\num{9.73(0.55)e-4} & \cellcolor{blue!5}\textbf{\num{6.08(1.21)e-3}} \\
& & \cellcolor{blue!5}CMT + LDR-L1 (ours) & \cellcolor{blue!5}\textbf{\num{-214.435(0.001)e0}} & \cellcolor{blue!5}\textbf{\num{98.81(0.09)}} & \cellcolor{blue!5}\textbf{\num{1.47(0.10)e-3}} & \cellcolor{blue!5}\textbf{\num{1.29(0.10)e-3}} & \cellcolor{blue!5}\textbf{\num{1.56(0.47)e-3}} & \cellcolor{blue!5}\textbf{\num{8.96(1.75)e-4}} & \cellcolor{blue!5}\num{8.37(1.48)e-3} \\
\cmidrule(lr){2-10} & \multirow{3}{*}{\num{1e8}} & CMT & \num{-214.431(0.000)e0} & \num{97.85(0.05)} & \textbf{\num{1.46(0.04)e-3}} & \num{1.39(0.03)e-3} & \textbf{\num{1.78(0.04)e-3}} & \textbf{\num{1.09(0.04)e-3}} & \textbf{\num{6.75(2.05)e-3}} \\
& & \cellcolor{blue!5}CMT + LDR-L2 (ours) & \cellcolor{blue!5}\textbf{\num{-214.437(0.001)e0}} & \cellcolor{blue!5}\num{99.11(0.06)} & \cellcolor{blue!5}\textbf{\num{1.39(0.04)e-3}} & \cellcolor{blue!5}\textbf{\num{1.36(0.03)e-3}} & \cellcolor{blue!5}\textbf{\num{1.94(0.42)e-3}} & \cellcolor{blue!5}\textbf{\num{1.11(0.08)e-3}} & \cellcolor{blue!5}\textbf{\num{7.65(2.06)e-3}} \\
& & \cellcolor{blue!5}CMT + LDR-L1 (ours) & \cellcolor{blue!5}\textbf{\num{-214.438(0.000)e0}} & \cellcolor{blue!5}\textbf{\num{99.29(0.04)}} & \cellcolor{blue!5}\textbf{\num{1.42(0.03)e-3}} & \cellcolor{blue!5}\textbf{\num{1.32(0.02)e-3}} & \cellcolor{blue!5}\textbf{\num{1.81(0.05)e-3}} & \cellcolor{blue!5}\textbf{\num{1.12(0.07)e-3}} & \cellcolor{blue!5}\textbf{\num{8.92(0.64)e-3}} \\
\midrule \multirow{3}{*}{} & \multirow{1}{*}{\num{4.2e8}} & FAB & \num{-504.355(0.019)e0} & \num{14.41(0.27)} & \num{2.15(0.12)e-2} & \num{1.02(0.05)e-2} & \num{3.62(0.06)e-2} & \num{1.57(0.07)e-2} & \num{3.38(1.25)e-2} \\
\multirow{3}{*}{} & \multirow{1}{*}{\num{4e8}} & TA-BG & \num{-504.782(0.019)e0} & \num{18.66(0.16)} & \num{6.74(1.29)e-3} & \num{6.74(0.89)e-3} & \num{8.19(1.17)e-3} & \num{1.14(0.10)e-2} & \num{1.60(0.49)e-2} \\
\cmidrule(lr){2-10}\multirow{11}{*}{\shortstack[c]{Alanine \\ Hexapeptide}} & \multirow{3}{*}{\num{1e8}} & CMT & \num{-503.190(0.051)e0} & \num{5.92(0.37)} & \num{6.72(0.73)e-2} & \num{5.23(0.51)e-2} & \num{1.56(0.45)e-1} & \textbf{\num{1.07(0.43)e-1}} & \textbf{\num{8.30(1.78)e-2}} \\
& & \cellcolor{blue!5}CMT + LDR-L2 (ours) & \cellcolor{blue!5}\textbf{\num{-504.801(0.008)e0}} & \cellcolor{blue!5}\textbf{\num{31.47(0.71)}} & \cellcolor{blue!5}\textbf{\num{1.61(0.08)e-2}} & \cellcolor{blue!5}\textbf{\num{1.56(0.23)e-2}} & \cellcolor{blue!5}\textbf{\num{1.33(0.06)e-2}} & \cellcolor{blue!5}\textbf{\num{1.58(0.08)e-2}} & \cellcolor{blue!5}\textbf{\num{5.92(1.28)e-2}} \\
& & \cellcolor{blue!5}CMT + LDR-L1 (ours) & \cellcolor{blue!5}\num{-504.721(0.010)e0} & \cellcolor{blue!5}\num{29.48(0.40)} & \cellcolor{blue!5}\textbf{\num{1.71(0.03)e-2}} & \cellcolor{blue!5}\textbf{\num{1.65(0.05)e-2}} & \cellcolor{blue!5}\textbf{\num{1.30(0.15)e-2}} & \cellcolor{blue!5}\textbf{\num{1.77(0.14)e-2}} & \cellcolor{blue!5}\textbf{\num{6.35(0.33)e-2}} \\
\cmidrule(lr){2-10} & \multirow{3}{*}{\num{2e8}} & CMT & \num{-503.545(0.364)e0} & \num{20.85(0.65)} & \num{9.08(3.16)e-2} & \textbf{\num{8.32(3.59)e-2}} & \num{1.89(0.19)e-1} & \num{1.41(0.24)e-1} & \textbf{\num{3.96(1.52)e-2}} \\
& & \cellcolor{blue!5}CMT + LDR-L2 (ours) & \cellcolor{blue!5}\num{-505.012(0.004)e0} & \cellcolor{blue!5}\num{41.71(0.47)} & \cellcolor{blue!5}\textbf{\num{9.05(0.65)e-3}} & \cellcolor{blue!5}\num{1.02(0.03)e-2} & \cellcolor{blue!5}\textbf{\num{8.81(0.50)e-3}} & \cellcolor{blue!5}\textbf{\num{1.12(0.01)e-2}} & \cellcolor{blue!5}\textbf{\num{3.32(0.54)e-2}} \\
& & \cellcolor{blue!5}CMT + LDR-L1 (ours) & \cellcolor{blue!5}\textbf{\num{-505.051(0.013)e0}} & \cellcolor{blue!5}\textbf{\num{43.66(0.69)}} & \cellcolor{blue!5}\textbf{\num{7.78(0.06)e-3}} & \cellcolor{blue!5}\textbf{\num{8.30(0.28)e-3}} & \cellcolor{blue!5}\textbf{\num{8.65(0.28)e-3}} & \cellcolor{blue!5}\textbf{\num{1.07(0.04)e-2}} & \cellcolor{blue!5}\textbf{\num{2.96(0.64)e-2}} \\
\cmidrule(lr){2-10} & \multirow{3}{*}{\num{4e8}} & CMT & \num{-504.770(0.037)e0} & \num{23.95(0.34)} & \num{1.39(0.07)e-2} & \num{1.27(0.06)e-2} & \num{1.54(0.08)e-2} & \num{1.88(0.05)e-2} & \num{6.17(0.55)e-2} \\
& & \cellcolor{blue!5}CMT + LDR-L2 (ours) & \cellcolor{blue!5}\textbf{\num{-505.193(0.007)e0}} & \cellcolor{blue!5}\textbf{\num{47.76(0.31)}} & \cellcolor{blue!5}\textbf{\num{5.08(0.12)e-3}} & \cellcolor{blue!5}\textbf{\num{5.52(0.06)e-3}} & \cellcolor{blue!5}\textbf{\num{6.53(0.07)e-3}} & \cellcolor{blue!5}\textbf{\num{7.07(0.10)e-3}} & \cellcolor{blue!5}\textbf{\num{1.22(0.19)e-2}} \\
& & \cellcolor{blue!5}CMT + LDR-L1 (ours) & \cellcolor{blue!5}\textbf{\num{-505.219(0.019)e0}} & \cellcolor{blue!5}\textbf{\num{48.32(0.66)}} & \cellcolor{blue!5}\textbf{\num{5.06(0.94)e-3}} & \cellcolor{blue!5}\textbf{\num{5.88(0.70)e-3}} & \cellcolor{blue!5}\textbf{\num{6.25(0.58)e-3}} & \cellcolor{blue!5}\textbf{\num{7.82(0.72)e-3}} & \cellcolor{blue!5}\textbf{\num{1.48(0.29)e-2}} \\
\bottomrule
\end{tabular}}
\end{sc}
\end{small}
\end{center}
\vskip -0.1in
\end{table*}

%% file: tables/results_tarflow_all.tex
\begin{table*}[!htbp]
\caption{Results for training on Cartesian coordinates using an unbiased dataset from the target distribution of alanine dipeptide.
We average over 4 independent experiments with standard deviations.
\textbf{Bold} indicates metrics not significantly different from the best under 
uncorrected two-sided Welch's t-tests ($\alpha = 0.05$), capturing seed-level variability only.}
\label{tab:results_tarflow_all}
\begin{center}
\begin{small}
\begin{sc}
\resizebox{0.7\textwidth}{!}{\begin{tabular}{lcccccc}
\toprule
Method & NLL $\downarrow$ & ESS / \% $\uparrow$ & RAM KL $\downarrow$ & RAM KL W. RW $\downarrow$ & $\mathcal{E}$-$\mathcal{W}_2$\\
\midrule FWD KL {\normalfont\cite{tanScalableEquilibriumSampling2025a}} & \num{-219.583(0.051)e0} & \num{19.40(0.80)} & \num{2.93(0.29)e-2} & \num{2.87(0.06)e-2} & \num{1.51(0.20)e-1} \\
FWD KL$^{\boldsymbol{\bigstar}}$ & \num{-219.583(0.050)e0} & \num{27.32(1.28)} & \num{2.93(0.30)e-2} & \num{2.51(0.08)e-2} & \num{1.43(0.09)e-1} \\
FWD KL$^{\boldsymbol{\bigstar}}$ (Path) & \num{-218.528(0.090)e0} & \num{10.31(0.86)} & \num{1.04(0.11)e-1} & \num{4.77(0.95)e-2} & \num{1.77(0.15)e-1} \\
\cellcolor{blue!5}FWD KL$^{\boldsymbol{\bigstar}}$ + LDR-L2 (ours) & \cellcolor{blue!5}\textbf{\num{-219.815(0.034)e0}} & \cellcolor{blue!5}\textbf{\num{35.89(1.69)}} & \cellcolor{blue!5}\textbf{\num{2.50(0.35)e-2}} & \cellcolor{blue!5}\textbf{\num{1.79(0.10)e-2}} & \cellcolor{blue!5}\textbf{\num{9.27(2.20)e-2}} \\
\cellcolor{blue!5}FWD KL$^{\boldsymbol{\bigstar}}$ + LDR-L1 (ours) & \cellcolor{blue!5}\textbf{\num{-219.800(0.044)e0}} & \cellcolor{blue!5}\textbf{\num{35.75(1.63)}} & \cellcolor{blue!5}\textbf{\num{2.16(0.08)e-2}} & \cellcolor{blue!5}\textbf{\num{2.07(0.23)e-2}} & \cellcolor{blue!5}\textbf{\num{1.15(0.33)e-1}} \\
\bottomrule
\end{tabular}}
\end{sc}
\end{small}
\end{center}
{\fontsize{7.5pt}{9pt}\selectfont $^{\boldsymbol{\bigstar}}$~Uses our improved augmentation correction, see Appendix~\ref{appendix:theory:com_augm}.}
\vskip -0.1in
\end{table*}